\documentclass{article}
\usepackage{iclr2027_conference,times}

\usepackage{amsmath,amsfonts,bm}

\def\eqref#1{equation~\ref{#1}}

\def\1{\bm{1}}

\DeclareMathAlphabet{\mathsfit}{\encodingdefault}{\sfdefault}{m}{sl}
\SetMathAlphabet{\mathsfit}{bold}{\encodingdefault}{\sfdefault}{bx}{n}

\newcommand{\R}{\mathbb{R}}

\usepackage{amsmath}
\usepackage{amssymb}
\usepackage{amsthm}
\usepackage{booktabs}
\usepackage{array}
\usepackage{graphicx}
\usepackage{algorithm}
\usepackage{algpseudocode}
\usepackage{tikz}
\usepackage{placeins}
\usepackage{xcolor}
\usepackage{hyperref}
\usepackage{url}
\hypersetup{hidelinks}

\newcommand{\method}{Convex-SKM-Net}

\newcommand{\metrichead}[2]{\shortstack{#1\\(#2)}}
\renewcommand{\eqref}[1]{Eq.~(\ref{#1})}

\newtheorem{theorem}{Theorem}[section]
\newtheorem{lemma}[theorem]{Lemma}
\newtheorem{proposition}[theorem]{Proposition}
\newtheorem{corollary}[theorem]{Corollary}

\title{Convex-SKM-Net: Low-Latency\\Block Halfspace Corrections for\\Neural Convex Constraint Satisfaction}

\iclrfinalcopy
\author{Haoyu Zhu\thanks{Equal contribution.}\\
Zhejiang University\\
\texttt{zhuhaoyu03@zju.edu.cn}
\And
Xiaochen Zhang\footnotemark[1]\\
Zhejiang University\\
\texttt{xiaochen.23@intl.zju.edu.cn}
\AND
Qingchun Hou\thanks{Corresponding author.}\\
Zhejiang University\\
\texttt{houqingchun@zju.edu.cn}}

\begin{document}

\maketitle

\begin{abstract}
Fast inference and strict constraint satisfaction are essential for learning-based decision-making. Existing methods often rely on constraint-specific projections or feasible-set parameterizations, while iterative repair can be costly. We introduce \method{}, a convex constraint-satisfaction layer that forms linearization halfspaces from residual and (sub)gradient oracles without constraint-specific inequality projectors. It applies fixed-depth dual projected-gradient descent (dual PGD) to minimum-displacement block problems using GPU-friendly matrix operations. We jointly learn the predictor and a positive-definite mobility, adapting the correction geometry to improve objective quality. We prove convergence of the outer iterates at fixed dual depth and linear convergence in distance under Slater's condition and bounded subgradients. All our test outputs on three second-order-cone model-predictive control (SOC-MPC) horizons and six power-system convex relaxations satisfy a $10^{-3}$ feasibility tolerance. Comparisons with the latest constraint-satisfaction methods show inference speedups of up to $144.3\times$ on SOC-MPC and $7.7\times$ on a convex quadratically constrained quadratic programming (QCQP) benchmark with joint training. On QCQP, the mean objective gap is $0.012\%$, with full test feasibility at $10^{-4}$. Learning mobility reduces the five-seed affine-task mean gap by $22.4\%$ and improves gaps in matched power-system comparisons.
\end{abstract}

\section{Introduction}
Learning-based decision-making in power systems, model-predictive control, and wireless resource allocation requires fast predictions that satisfy physical constraints and retain good objective values~\citep{donti2021dc3,chen2018constrainedmpc,alizadeh2023powercontrol}. Neural predictions can violate constraints even when violations are penalized during training~\citep{zeng2024glinsat,donti2021dc3}. Constraint layers enforce feasibility through feasible-set parameterizations, optimization, or iterative repair. Their use depends on the operations available for each constraint family: parameterizations may require interior points~\citep{tordesillas2023rayen}, and splitting methods require component projections~\citep{grontas2026pinet,tang2026lminet}. Residual-based repair works directly with constraint values and derivatives, but repeated correction can remain costly~\citep{donti2021dc3,nguyen2025fsnet,chu2026snarenet}. Repair also changes the predicted decision, so its direction affects the final objective. We therefore design both the correction computation and its directional cost.

\method{} uses convexity to give different inequalities a common local representation. A residual and a subgradient define a linearization halfspace containing the feasible set, allowing constraints to be handled without family-specific projectors. After satisfying the affine equalities, we formulate minimum-displacement problems over selected blocks of these halfspaces. Fixed-depth dual projected-gradient descent (dual PGD) computes block updates with GPU-friendly matrix operations, without per-block factorizations or linear solves. Several block corrections reuse one linearization. On second-order-cone model-predictive control (SOC-MPC), inference is $80.7$--$144.3\times$ faster than fixed-$400$-step $\Pi$net. On the 300-bus semidefinite relaxation, postprocessing a fully fine-tuned predictor is $7.7\times$ faster than LMI-Net, with a lower mean gap on identical raw predictions.

We also learn how displacement is measured during repair. Euclidean distance assigns the same cost to equal-length moves, although their effects on the task objective can differ. A positive-definite mobility changes these directional costs while preserving the feasible set. Joint training adjusts both the predictor's initial output and the geometry of its subsequent corrections. The same layer supports postprocessing with a prescribed mobility, including the identity. Learning mobility reduces the five-seed mean gap on the affine task by $22.4\%$. In matched joint-training comparisons, it also reduces mean gaps by $11.0\%$ and $8.1\%$ on the 300-bus conic and semidefinite relaxations. Our contributions are:

\noindent\textbf{(1) General convex constraint satisfaction.} We introduce a layer using residual and (sub)gradient oracles without constraint-specific inequality projectors, supporting both postprocessing and joint training.

\noindent\textbf{(2) Block corrections with dual PGD.} Our GPU-friendly block updates use fixed dual depth without per-block factorizations. We establish convergence of the outer iterates and linear convergence in distance under Slater's condition and bounded subgradients.

\noindent\textbf{(3) Learnable mobility.} We learn directional correction costs jointly with the predictor, improving objective gaps on affine and nonlinear convex tasks.

\section{Related Work}
We review neural constraint layers and the numerical methods underlying our corrections.

\paragraph{Feasible-set parameterization.}
Linear-constraint constructions use polyhedral generators, gauge maps, or decision rules~\citep{frerix2020homogeneous,li2023hardlinear,constante2025decisionrules}; other approaches calibrate constraints or restrict network maps over prescribed input regions~\citep{zhao2023preventive,balestriero2023police}. RAYEN uses an interior point and ray boundaries, computing the point offline for fixed constraints~\citep{tordesillas2023rayen}. Radial transformations, invertible maps, and interior-point predictions also construct constrained outputs~\citep{konstantinov2023output,liang2024homeomorphic,liang2025bisection,schneider2026softradial}.

\paragraph{Differentiable solver layers.}
Affine repair uses closed-form updates or subconstraint decompositions~\citep{min2025hardnet,zhao2026caffnet}. Optimization layers differentiate through constrained solutions during training~\citep{amos2017optnet,agrawal2019differentiable,min2025hardnet}. Linear constraints admit generalized Sinkhorn and entropy-regularized dual iterations, with linear repair also used in tabular generative models~\citep{wang2023linsatnet,zeng2024glinsat,stoian2024cdgm}. $\Pi$net and LMI-Net use component projections in operator splitting~\citep{grontas2026pinet,tang2026lminet}; Dykstra's algorithm and alternating projections also support constrained prediction~\citep{chen2018constrainedmpc,cristian2023projectnet}.

Among residual-based layers, DC3 combines equality completion with violation gradients, and FSNet minimizes violation~\citep{donti2021dc3,nguyen2025fsnet}. Other updates use adaptive repair, damped linearization, optimality conditions, or regularized Newton steps~\citep{lastrucci2025enforce,goertzen2026hardnetpp,iftakher2026kkthardnet,chu2026snarenet}.

Sampling Kaczmarz--Motzkin (SKM) corrects the most violated sampled row~\citep{deloera2017sampling}; nonlinear SKM handles nonlinear equations and variants with convex constraints~\citep{zhangSamplingKaczmarzMotzkin2023}. T-SKM-Net unrolls row repair; AT-SKM-Net uses graph-guided sampling for dynamic linear constraints~\citep{zhu2026tskm,zhang2026atskm}. Linear block methods project onto equalities or fit boundaries~\citep{needell2014paved,zhang2021block,zhang2023randomized,briskman2015block}; subgradient, parallel, block-iterative, and minibatch methods also address convex inequalities~\citep{dossantos1987parallel,kiwiel1995block,bauschke2017convex,necoara2019minibatch}.

\section{Method: \method{}}\label{sec:method}
\textbf{Overview.} Figure~\ref{fig:pipeline} illustrates the layer $y_{\rm corr}=\mathcal R(x,\hat y;P)$, which takes an instance $x$, a prediction $\hat y$, and a prescribed positive-definite mobility $P$. It first satisfies the affine equalities, then restricts corrections to their nullspace. Each outer iteration builds linearization halfspaces from residuals and (sub)gradients. The layer selects a block using normalized violation scores and applies $T$ dual PGD steps to its minimum-displacement problem. It performs $S$ block corrections before refreshing the linearization. Mobility controls normalization and directional costs and stays fixed during one repair. A coordinate transformation reuses the Euclidean dual solver. Algorithm~\ref{alg:layer} gives the full procedure. Postprocessing supplies $P$ directly; joint training learns its parameters together with the predictor.

For input $x$, a predictor returns $\hat y=f_\theta(x)\in\mathbb R^n$. Assume the nonempty feasible set $\mathcal Y(x)=\{y:C(x)y=d(x),\ g_i(y,x)\le0,\ i=1,\ldots,m\}$ has closed convex inequalities with queryable residuals and subgradients, and projectable affine equalities. Set $y_{\rm eq}=P_{\{Cy=d\}}(\hat y)$ and take an orthonormal basis $N$ of $\ker C$, so $CN=0$ and $N^\top N=I_q$. Then
\begin{equation}\label{eq:reduced-setting}
y=y_{\rm eq}+Nz,\qquad h_i(z)=g_i(y_{\rm eq}+Nz,x),\qquad
\Omega=\{z:h_i(z)\le0,\ \forall i\}.
\end{equation}
Input dependence is omitted below. Appendix~\ref{app:oracles} gives the residual and (sub)gradient oracles for each task.

\begin{figure}[!htbp]
\centering
\input{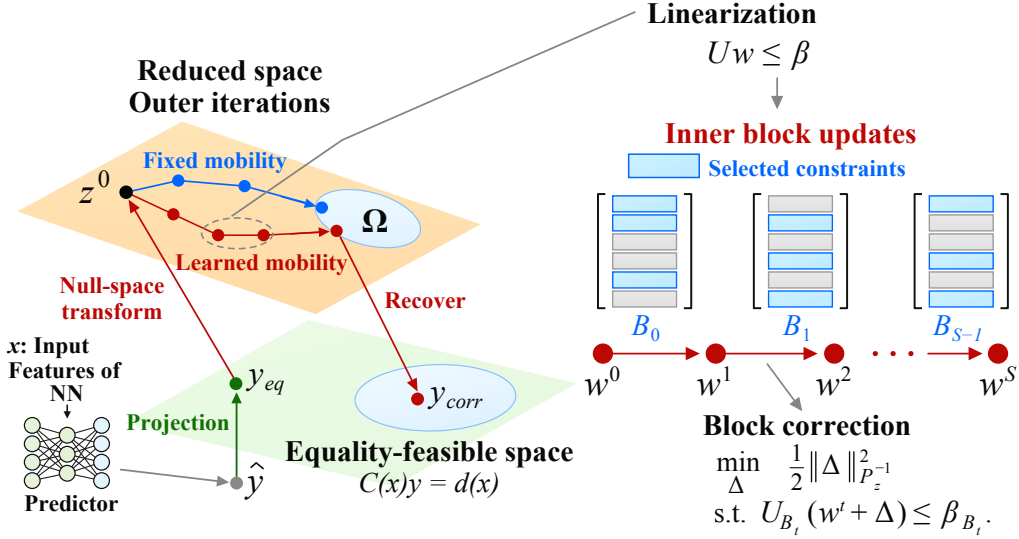}
\caption{Repair in reduced space with fixed (blue) or learned (red) mobility. The inset shows block reselection under one linearization. Trajectories are schematic.}\label{fig:pipeline}
\end{figure}

\subsection{Halfspace block corrections}\label{sec:single-block}
At $z^k$, take $a_i^k\in\partial h_i(z^k)$. Convexity gives $h_i(z)\ge\ell_i^k(z)=h_i(z^k)+\langle a_i^k,z-z^k\rangle$. Hence the linearization halfspace contains the feasible set, with single-row SKM correction
\begin{equation}
H_i^k=\{z:\ell_i^k(z)\le0\},\quad \Omega\subseteq H_i^k,\qquad
\Delta_i^k=-\frac{[h_i(z^k)]_+}{\|a_i^k\|^2}a_i^k.
\end{equation}
Satisfied rows make no update; violated rows have nonzero subgradients (Lemma~\ref{lem:nonzero}). We apply this construction to a block of constraints.

For nonzero rows define
\[
u_i^k=\frac{a_i^k}{\|a_i^k\|},\qquad
\beta_i^k=\frac{(a_i^k)^\top z^k-h_i(z^k)}{\|a_i^k\|},
\]
so $H_i^k=\{z:(u_i^k)^\top z\le\beta_i^k\}$. Zero-subgradient rows use zero placeholders. At inner iterate $w$, initially $z^k$, select the $b$ largest normalized violation scores $\delta_i^k(w)=[(u_i^k)^\top w-\beta_i^k]_+$, with $1\le b\le m$.
Normalization removes positive row scaling; zero-score rows fill the block when fewer than $b$ rows are violated. Stack the selected rows in $U_B$, the offsets in $\beta_B$, and set $s=U_Bw-\beta_B$. The block problem retains these signed residuals and minimizes displacement into the halfspace intersection:
\begin{equation}\label{eq:block_primal}
\min_\Delta\tfrac12\|\Delta\|^2\quad\text{s.t.}\quad U_B\Delta+s\le0.
\end{equation}
Containment makes this quadratic program (QP) feasible. For $b=1$, \eqref{eq:block_primal} gives the projection step from $w$ onto the selected halfspace.

\subsection{Dual PGD}\label{sec:dual}
Abbreviate $U=U_B$ and $G_B=UU^\top$. The dual objective and updates are
\begin{equation}\label{eq:dual-update}
\begin{aligned}
\max_{\lambda\ge0}\phi(\lambda)&=s^\top\lambda-\tfrac12\|U^\top\lambda\|^2,\\
\lambda^{j+1}&=\left[\lambda^j+\frac{s-G_B\lambda^j}{L_B}\right]_+,\qquad
\Delta=-U^\top\lambda^T.
\end{aligned}
\end{equation}
Each block takes $T\ge1$ steps from $\lambda^0=0$, reusing $G_B$. The dual variables weight row normals to give the displacement. We take $\lambda_{\max}(G_B)\le L_B\le\bar L<\infty$, with coefficients and steps fixed within each block; zero blocks stay still. Normalization permits $\bar L=b$; Appendix~\ref{app:gershgorin} covers broader steps.

Given residuals and (sub)gradients, block updates use matrix products, reductions, and elementwise clipping, without per-block factorizations or linear solves. Fixed block shapes support GPU batching; complexity is in Appendix~\ref{app:complexity}.

\label{sec:block-operations}
\begin{table}[!htbp]
\caption{Local operations on selected linearizations. Boundary fitting selects a minimum-norm least-squares displacement.}\label{tab:block_comparison}
\centering\setlength{\tabcolsep}{3pt}
\begin{tabular}{@{}p{1.0in}p{2.45in}p{1.78in}@{}}
\toprule
Operation & Local objective & Computation\\\midrule
Single-row SKM & Minimum displacement into one halfspace & Closed-form projection\\
Equality-LS & Boundary fit on all selected rows & Pseudoinverse or linear solve\\
Block-pinv & Boundary fit on violated rows & Pseudoinverse\\
Violation-GD & $\tfrac12\|[U_B\Delta+s]_+\|^2$ & Primal gradient steps\\
Dual PGD (ours) & Minimum displacement under inequalities & Dual projected-gradient steps\\\bottomrule
\end{tabular}
\end{table}
Boundary fitting targets hyperplanes~\citep{needell2014paved,briskman2015block}, while violation descent minimizes the squared positive residual~\citep{donti2021dc3}. Our block problem minimizes displacement under halfspace inequalities.

\subsection{Learnable mobility}\label{sec:mobility}
We introduce a positive-definite mobility $P$ to control directional displacement costs. Joint training learns $P=P_\psi(x)$ together with the predictor $f_\theta$, improving the initial prediction and subsequent repair geometry. Postprocessing prescribes $P$; $P=I_n$ recovers Euclidean repair.

We use $P=\operatorname{diag}(p_j)\succ0$ and $P_z=N^\top PN\succ0$, fixed during one complete repair. Let $A_B$ stack selected unnormalized rows $(a_i^k)^\top$ and $s_B^{\rm raw}=(\ell_i^k(w))_{i\in B}$ contain their signed linearized residuals. We replace~\eqref{eq:block_primal} with
\begin{equation}\label{eq:mobility_primal}
\min_{\Delta z}\tfrac12\Delta z^\top P_z^{-1}\Delta z
\quad\text{s.t.}\quad A_B\Delta z+s_B^{\rm raw}\le0.
\end{equation}
Larger mobility eigenvalues lower the cost of the corresponding directions. Selection uses $[\ell_i^k(w)]_+/\sqrt{(a_i^k)^\top P_za_i^k}$, and the feasible set is unchanged. The rectangular factor $R=N^\top P^{1/2}\in\mathbb R^{q\times n}$ satisfies $RR^\top=P_z$. Setting $z=R\nu$, with $\nu\in\mathbb R^n$, reuses the Euclidean solver (Algorithm~\ref{alg:layer}).\label{sec:training} Parameterization and training are in Appendix~\ref{app:training-protocol}; finite-step equivalence is in Appendix~\ref{app:mobility-theory}.

\begin{algorithm}[!htbp]
\caption{\method{} repair with prescribed mobility}\label{alg:layer}
\begin{algorithmic}[1]
\Require $x,\hat y$, mobility $P=\operatorname{diag}(p)\succ0$, block size $b$, outer cap $K$, inner steps $S\ge1$, dual depth $T\ge1$, tolerance $\varepsilon$
\State Obtain $y_{\rm eq}=P_{\{Cy=d\}}(\hat y)$ and an orthonormal nullspace basis $N$
\State $R\gets N^\top P^{1/2}$; $\nu\gets0$
\State Define $\tilde h_i(\nu)=g_i(y_{\rm eq}+NR\nu,x)$
\For{$k=0,\ldots,K$}
\State $y\gets y_{\rm eq}+NR\nu$; $r\gets v(x,y)$
\State \textbf{if} $r\le\varepsilon$ or $k=K$ \textbf{then return} $y,r,(r\le\varepsilon)$
\State Linearize and normalize $\tilde h$ at $\nu$ to obtain $(\bar U,\bar\beta)$; $\tilde w\gets\nu$
\For{$t=1,\ldots,S$}
\State $s_{\rm all}\gets\bar U\tilde w-\bar\beta$; $B\gets\operatorname{top}_b([s_{\rm all}]_+)$
\State $(U,s)\gets(\bar U_B,(s_{\rm all})_B)$
\State $\lambda\gets$ \textsc{DualPGD}$(U,s,T)$ starting at zero
\State $\tilde w\gets\tilde w-U^\top\lambda$
\EndFor
\State $\nu\gets\tilde w$
\EndFor
\end{algorithmic}
\end{algorithm}
Here $v(x,y)$ is the maximum absolute equality residual or positive inequality residual in the original coordinates; success means $v\le\varepsilon$. \textsc{DualPGD} uses the update form in~\eqref{eq:dual-update}; task-specific step sizes are in Appendix~\ref{app:training-protocol}.

\section{Theoretical Guarantees}\label{sec:theory}
We analyze how fixed-depth block corrections drive the outer iterates to feasibility.
Fix an input with nonempty feasibility, finite convex $h_i$, and violated-row subgradients bounded by $M$ on the iterate region. We use Algorithm~\ref{alg:layer}'s zero initialization, $1/L_B$ steps, and full displacements, in exact arithmetic with zero tolerance and no outer cap; feasible outputs stay fixed. Fixed mobility first transforms to Euclidean coordinates (Appendix~\ref{app:mobility-theory}); distances and constants below refer to these coordinates.

For a block, let $\delta_B=[s]_+$, $E_B=\|\delta_B\|^2$, $G_B=UU^\top$, and $c_{B,T}=2\phi(\lambda^T)$. For $E_B>0$, also set $r_B=\delta_B^\top G_B\delta_B/E_B$.
\begin{theorem}[Progress at finite dual depth]\label{thm:inexact-main}
For $E_B>0$, $T\ge1$, and any $z^\star\in\Omega$, the update $z^+=z-U^\top\lambda^T$ satisfies
\begin{equation}
\begin{aligned}
\|z^+-z^\star\|^2&\le\|z-z^\star\|^2-c_{B,T},\\
c_{B,T}&\ge E_B\max\left\{L_B^{-1},(r_B+L_B/T)^{-1}\right\}.
\end{aligned}
\end{equation}
If $E_B=0$, then $\lambda^T=0$, $c_{B,T}=0$, and $z^+=z$. Any fixed $T\ge1$ therefore permits updates without an inner convergence test.
\end{theorem}
\textit{Proof sketch.}
Halfspace containment gives a squared-distance decrease of at least $2\phi(\lambda^T)$. The first PGD step and dual monotonicity yield $c_{B,T}\ge E_B/L_B$. Applying the PGD comparison bound to $\mu=a\delta_B$ and optimizing over $a\ge0$ gives the second term. Appendix~\ref{app:pgd-depth} gives the proof.

\begin{corollary}[Convergence and linear rate]\label{cor:progress}
With fixed $S,T\ge1$, the outer iterates converge to a point in $\Omega$. If $z^0\notin\Omega$ and there exist $\bar z$ and $\sigma>0$ such that $h_i(\bar z)\le-\sigma$ for all $i$, then the distance $d_k=\operatorname{dist}(z^k,\Omega)$ satisfies
\begin{equation}
d_K^2\le(1-\gamma)^Kd_0^2,\qquad
\gamma=\frac{b}{m\bar L\kappa_2^2},\quad
\kappa_2=\frac{M\|z^0-\bar z\|}{\sigma}.
\end{equation}
For an infeasible start, $0<\gamma\le1$; a feasible start stays fixed.
\end{corollary}
\textit{Proof sketch.}
Let $\Psi_k=\sum_i(\delta_i^k(z^k))^2$. The first top-$b$ block gives $d_{k+1}^2\le d_k^2-b\Psi_k/(m\bar L)$. Fej\'er monotonicity, bounded subgradients, and continuity imply feasible-point convergence. Under Slater's condition, $d_k^2\le\kappa_2^2\Psi_k$ yields the stated linear rate. Further inner updates preserve containment and add descent. Appendices~\ref{app:inner-composition}--\ref{app:block-gain} give the proofs.

\section{Experiments}\label{sec:experiments}
We evaluate postprocessing and joint training across convex constraint families.

\begin{samepage}
Gap averages per-instance percentage differences from reference objectives, retaining signs for second-order-cone model-predictive control (SOC-MPC) and clipping negatives elsewhere. Feasibility uses maximum original violation; statistics include failures. Neural inference uses RTX 4090 GPUs, full solvers CPUs. Appendix~\ref{app:training-protocol} specifies data, training, timing, and repair budgets. In all curves, only $\Pi$net's inference depth $K$ varies; checkpoints and other settings stay fixed, and every input takes $K$ steps.\par
\end{samepage}

\subsection{Quadrotor SOC-MPC}\label{sec:soc-mpc}
We predict state and control sequences for a six-state, two-input quadrotor model with linear dynamics and second-order-cone constraints, using horizons $H=20,40,80$. Neural methods train without labels for 100 epochs. Ours, $\Pi$net, and SnareNet share initial predictor weights and sample order. Table~\ref{tab:soc_mpc} uses $\Pi$net at $K=400$, the smallest evaluated fixed depth yielding full test feasibility at every horizon; Appendix~\ref{app:soc-mpc-controls} gives the full scan and original validation-selected depths.

\begin{table}[!htb]
\caption{SOC-MPC: $1{,}000$ tests per horizon; mean signed Gap and mean single-instance time. Bold marks best neural Gap/time. Gurobi/MOSEK use single-thread CPUs. ---: SnareNet training exceeded 24 GB on RTX 4090 with default XLA.}
\label{tab:soc_mpc}
\centering\setlength{\tabcolsep}{1pt}
\begin{minipage}[t]{0.323\linewidth}
\centering
\textbf{(a) $H=20$}\par
\begin{tabular*}{\linewidth}{@{\extracolsep{\fill}}lrrr@{}}
\toprule
Method & \metrichead{Gap}{\%} & \metrichead{Feas.}{\%} & \metrichead{Time}{ms} \\
\midrule
Gurobi & 0.000 & 100 & 20.809 \\
MOSEK & 0.000 & 100 & 4.974 \\
DC3 & 5.909 & 100 & \textbf{0.566} \\
$\Pi$net & \textbf{0.130} & 100 & 75.952 \\
SnareNet & 2.913 & 100 & 58.567 \\
Ours & 0.192 & 100 & 0.941 \\
\bottomrule
\end{tabular*}
\end{minipage}\hfill
\begin{minipage}[t]{0.323\linewidth}
\centering
\textbf{(b) $H=40$}\par
\begin{tabular*}{\linewidth}{@{\extracolsep{\fill}}lrrr@{}}
\toprule
Method & \metrichead{Gap}{\%} & \metrichead{Feas.}{\%} & \metrichead{Time}{ms} \\
\midrule
Gurobi & 0.000 & 100 & 57.957 \\
MOSEK & 0.000 & 100 & 10.023 \\
DC3 & 16.696 & 100 & 1.810 \\
$\Pi$net & 0.567 & 100 & 153.111 \\
SnareNet & 1.762 & 100 & 175.515 \\
Ours & \textbf{0.236} & 100 & \textbf{1.061} \\
\bottomrule
\end{tabular*}
\end{minipage}\hfill
\begin{minipage}[t]{0.323\linewidth}
\centering
\textbf{(c) $H=80$}\par
\begin{tabular*}{\linewidth}{@{\extracolsep{\fill}}lrrr@{}}
\toprule
Method & \metrichead{Gap}{\%} & \metrichead{Feas.}{\%} & \metrichead{Time}{ms} \\
\midrule
Gurobi & 0.000 & 100 & 59.254 \\
MOSEK & 0.000 & 100 & 20.470 \\
DC3 & 13.346 & 99.4 & 6.623 \\
$\Pi$net & 0.230 & 100 & 304.138 \\
SnareNet & --- & --- & --- \\
Ours & \textbf{0.162} & 100 & \textbf{3.125} \\
\bottomrule
\end{tabular*}
\end{minipage}
\end{table}

All $3{,}000$ test outputs from \method{} satisfy the $10^{-3}$ tolerance, with mean gaps below $0.240\%$ and inference times of $0.94$--$3.13$ ms. This gives an $80.7$--$144.3\times$ speedup over fixed-$400$-step $\Pi$net. $\Pi$net has a lower gap at $H=20$, while \method{} has lower gaps at $H=40$ and $H=80$. Figure~\ref{fig:soc_mpc_quality_latency} shows that $\Pi$net can attain small mean gaps well before all outputs meet the feasibility tolerance.

\begin{figure}[!htb]
\centering
\includegraphics[width=\linewidth]{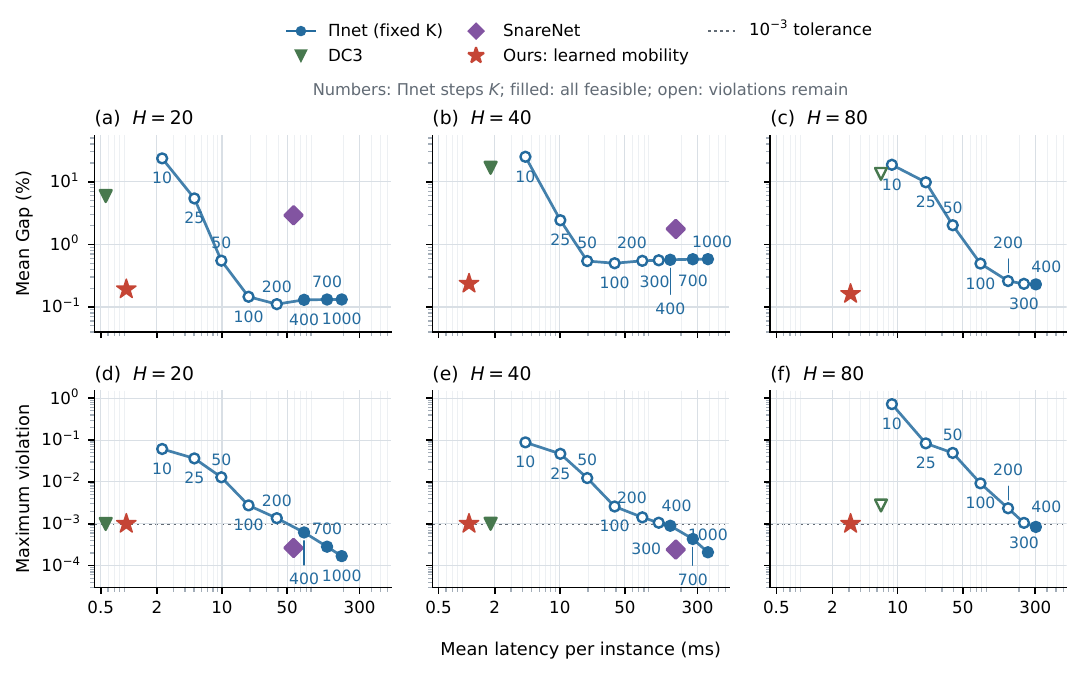}
\caption{Fixed-depth SOC-MPC: columns $H=20,40,80$; mean signed Gap (top) and maximum original violation (bottom) versus mean latency. Each point covers $1{,}000$ tests at tolerance $10^{-3}$. Other methods: Table~\ref{tab:soc_mpc} settings. Full values: Appendix~\ref{app:soc-mpc-controls}.}
\label{fig:soc_mpc_quality_latency}
\end{figure}

\subsection{Power-system convex relaxations}\label{sec:acopf}
We evaluate second-order-cone programming (SOCP) and semidefinite programming (SDP) relaxations of alternating-current optimal power flow (ACOPF) on 57-, 118-, and 300-bus systems. The cases contain $1{,}004$--$5{,}658$ inequalities, with $200$, $400$, and $1{,}000$ test instances at the three sizes. Three variants fine-tune (FT) a common pretrained predictor for $100$ epochs: full FT post updates the entire predictor, head FT post only its solution head, both with identity-mobility repair; joint freezes the trunk and learns solution and mobility heads. External postprocessors share raw predictions only with full FT post. Joint training supervises both raw and corrected predictions. Feasibility is evaluated at tolerance $10^{-3}$; Appendix~\ref{app:training-protocol} gives budgets and stopping rules.

\begin{table}[!htb]
\caption{SOCP/SDP: mean Gap and mean single-instance time, including failures. External postprocessors share full FT post's raw predictions. Bold compares our variants, including rounded ties; ---: case-300 SnareNet was too slow to evaluate fully.}
\label{tab:acopf_socp}\label{tab:acopf_sdp}
\centering\setlength{\tabcolsep}{1.5pt}
\begin{tabular}{@{}l rrr@{\hspace{8pt}}rrr@{\hspace{8pt}}rrr@{}}\toprule
Method & \multicolumn{3}{c}{57 buses} & \multicolumn{3}{c}{118 buses} & \multicolumn{3}{c}{300 buses}\\
\cmidrule(lr){2-4}\cmidrule(lr){5-7}\cmidrule(lr){8-10}
 & \metrichead{Gap}{\%} & \metrichead{Feas.}{\%} & \metrichead{Time}{ms} & \metrichead{Gap}{\%} & \metrichead{Feas.}{\%} & \metrichead{Time}{ms} & \metrichead{Gap}{\%} & \metrichead{Feas.}{\%} & \metrichead{Time}{ms}\\\midrule
\multicolumn{10}{c}{\textbf{(a) SOCP}}\\

MOSEK & 0.000 & 100 & 31.234 & 0.000 & 100 & 108.978 & 0.000 & 100 & 410.376 \\
$\Pi$net & 0.059 & 100 & 203.366 & 0.396 & 98.5 & 2169.503 & 0.106 & 99.2 & 20934.989 \\
FSNet & 0.053 & 100 & 26.556 & 0.395 & 97.0 & 355.968 & 0.112 & 8.5 & 3199.870 \\
SnareNet & 0.116 & 100 & 48.862 & 0.530 & 88.25 & 6452.791 & --- & --- & --- \\
Ours-post (full) & 0.051 & 100 & 0.579 & 0.419 & 100 & 3.279 & 0.098 & 100 & 41.676 \\
Ours-post (head) & 0.051 & 100 & \textbf{0.571} & 0.184 & 100 & \textbf{3.160} & 0.041 & 100 & 35.788 \\
Ours-joint & \textbf{0.050} & 100 & 0.676 & \textbf{0.167} & 100 & 3.386 & \textbf{0.037} & 100 & \textbf{22.672} \\
\midrule\multicolumn{10}{c}{\textbf{(b) SDP}}\\

MOSEK & 0.000 & 100 & 192.095 & 0.000 & 100 & 357.923 & 0.000 & 100 & 991.394 \\
LMI-Net & 0.076 & 100 & 14.768 & 0.658 & 100 & 93.379 & 0.086 & 100 & 2515.586 \\
SnareNet & 0.163 & 87.5 & 9831.992 & 0.706 & 26.25 & 60196.882 & --- & --- & --- \\
Ours-post (full) & 0.075 & 100 & 5.522 & 0.373 & 100 & 17.659 & 0.052 & 100 & 328.205 \\
Ours-post (head) & \textbf{0.064} & 100 & \textbf{4.614} & 0.165 & 100 & 13.780 & 0.054 & 100 & 188.657 \\
Ours-joint & \textbf{0.064} & 100 & 4.933 & \textbf{0.144} & 100 & \textbf{12.744} & \textbf{0.038} & 100 & \textbf{183.702} \\
\bottomrule\end{tabular}
\end{table}

All three variants in Table~\ref{tab:acopf_socp} satisfy the $10^{-3}$ tolerance on all six ACOPF benchmarks. On SDP300, full FT post yields a lower mean gap than LMI-Net with a $7.7\times$ inference speedup on identical raw predictions. Joint training with learned mobility matches or improves on head FT post in mean gap on every case.

With the same head-only joint-training scope, task loss, and repair budget, adding mobility reduces mean gap by $11.0\%$ on SOCP300 and $8.1\%$ on SDP300 relative to identity mobility. Appendices~\ref{app:acopf-mobility}--\ref{app:dc3-acopf} give full controls, failure residuals, and DC3 evaluations.

\FloatBarrier
\subsection{Nonconvex optimization with affine constraints}\label{sec:affine}
The public $\Pi$net task~\citep{grontas2026pinet} has $100$ variables, $50$ equalities, $50$ linear inequalities, and sine terms in its objective. Five seeds each test all $1{,}024$ instances. Methods train for $50$ epochs except DC3, which retains its official $1{,}000$-epoch setting; JAXopt uses differentiable OSQP projection. Our two variants differ only in learning mobility.

\begin{table}[!htb]
\caption{Affine task: five-seed means, tolerance $10^{-3}$. Single-instance latency averages seed timing medians. Identity/learned refers to mobility. Full statistics: Appendix~\ref{app:pinet31-controls}.}\label{tab:pinet31}
\centering\setlength{\tabcolsep}{4pt}
\begin{tabular*}{\linewidth}{@{\extracolsep{\fill}}lrrrrrr@{}}\toprule
 & DC3 & JAXopt & SnareNet & $\Pi$net ($K=50$) & Ours (identity) & Ours (learned)\\\midrule
Gap (\%) & 7.894 & 0.388 & 6.169 & 0.331 & 0.245 & \textbf{0.190} \\
Feas. (\%) & 99.824 & 100 & 100 & 100 & 100 & 100 \\
Time (ms) & 3.043 & 7.226 & 3.041 & 0.642 & \textbf{0.410} & 0.523 \\
\bottomrule\end{tabular*}
\end{table}

Table~\ref{tab:pinet31} shows that learning mobility reduces the five-seed mean gap from $0.245\%$ to $0.190\%$, with improvements in four seeds. Figure~\ref{fig:affine_quality_latency} shows a lower gap than the fully feasible $K=40$ $\Pi$net configuration at comparable single-instance latency, $0.523$ versus $0.533$ ms.

\begin{figure}[!htb]
\centering
\includegraphics[width=0.92\linewidth]{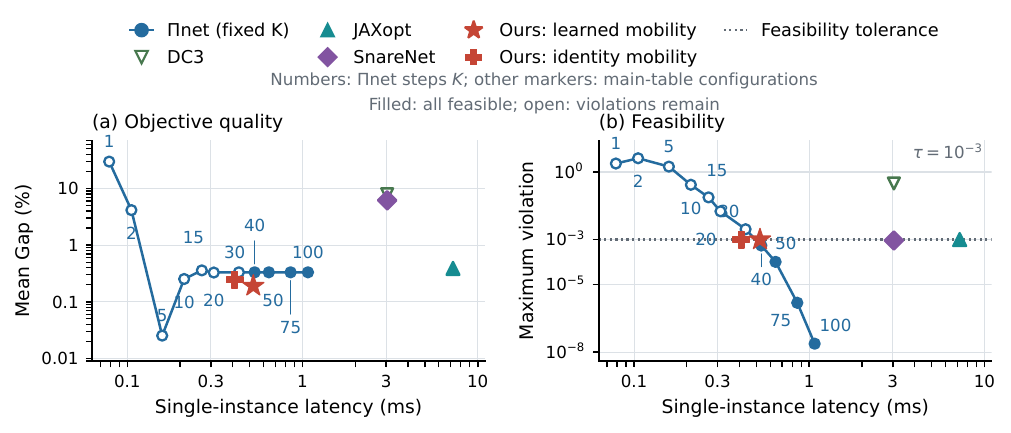}
\caption{Affine-task fixed-depth comparison: mean nonnegative Gap (left) and maximum violation over $5{,}120$ outputs (right) versus latency; both axes are logarithmic. Tolerance is $10^{-3}$. Gap and latency average five seeds. Full values: Appendix~\ref{app:pinet31-controls}.}
\label{fig:affine_quality_latency}
\end{figure}

\FloatBarrier
\subsection{Convex quadratically constrained optimization}\label{sec:qcqp}
SnareNet's public convex quadratically constrained quadratic program (QCQP)~\citep{chu2026snarenet} has $100$ variables, $100$ quadratic inequalities, and $50$ affine equalities. All methods train for $1{,}500$ epochs with one training seed and test all $833$ instances. Joint models train through their layers; our postprocessor uses identity mobility. SnareNet retains AdaRel; ours uses original constraints. $\Pi$net converts quadratic inequalities to SOCs; its tabulated joint/post depths and DC3's test repair are validation-selected.

\begin{table}[!htb]
\caption{QCQP: all seven configurations achieve $100\%$ feasibility at $10^{-4}$. Mean nonnegative Gap and prediction-plus-repair latency, averaged over three complete test passes. Bold marks best Gap/time per mode, including rounded ties.}\label{tab:qcqp}
\centering\setlength{\tabcolsep}{4pt}
\begin{minipage}[t]{0.47\linewidth}
\vspace{0pt}\centering
\textbf{(a) Joint}\par
\begin{tabular*}{\linewidth}{@{\extracolsep{\fill}}lrr@{}}\toprule
Method & Gap (\%) & Time (ms)\\\midrule
DC3 & 66.918 & \textbf{0.600} \\
SnareNet & 0.159 & 18.963 \\
$\Pi$net ($K=200$) & 0.040 & 6.126 \\
Ours & \textbf{0.012} & 2.469 \\
\bottomrule\end{tabular*}
\end{minipage}\hfill
\begin{minipage}[t]{0.47\linewidth}
\vspace{0pt}\centering
\textbf{(b) Post}\par
\begin{tabular*}{\linewidth}{@{\extracolsep{\fill}}lrr@{}}\toprule
Method & Gap (\%) & Time (ms)\\\midrule
SnareNet & 3.963 & 13.846 \\
$\Pi$net ($K=400$) & \textbf{3.942} & 11.880 \\
Ours & \textbf{3.942} & \textbf{0.617} \\
\bottomrule\end{tabular*}
\end{minipage}
\end{table}

With joint training, \method{} returns feasible outputs on all $833$ test inputs and is $7.7\times$ faster than SnareNet with a lower mean gap (Table~\ref{tab:qcqp}). It also improves on the fully feasible $K=100$ $\Pi$net configuration in both mean gap ($0.012\%$ versus $0.040\%$) and latency ($2.469$ versus $3.175$ ms).

Figure~\ref{fig:qcqp_quality_latency} compares three postprocessors sharing one soft-constraint predictor. Ours attains a $3.942\%$ mean gap, comparable to fixed-$400$-step $\Pi$net, with $0.617$ ms inference and a $22.4\times$ speedup over SnareNet.

\begin{figure}[!t]
\centering
\includegraphics[width=0.90\linewidth]{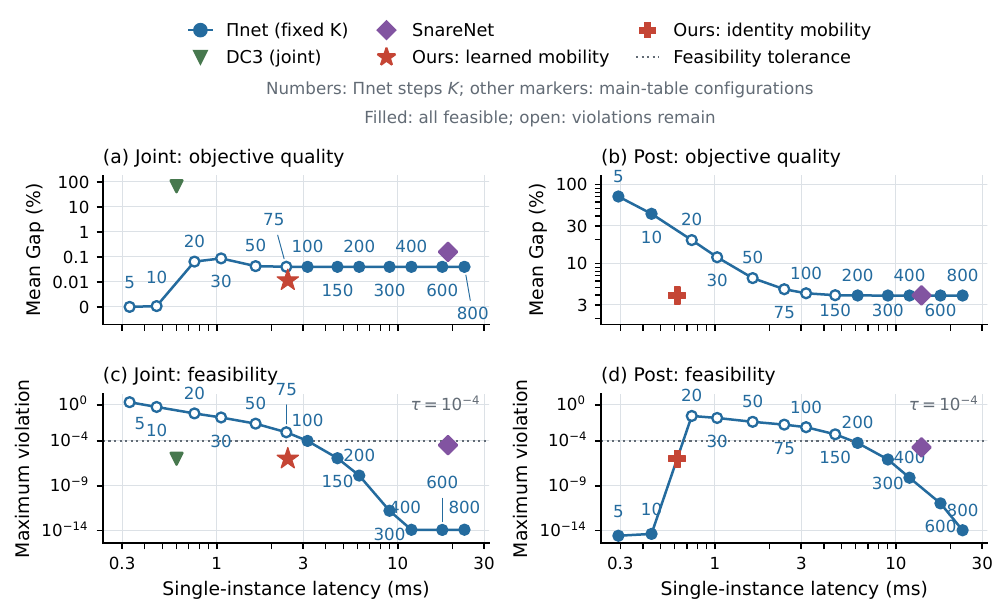}
\caption{QCQP fixed-depth sweep: joint/post columns, mean Gap/maximum original violation rows, and mean latency; all $833$ tests per point at tolerance $10^{-4}$. Axes are logarithmic except joint Gap on $[0,0.01\%]$. Full values: Appendix~\ref{app:qcqp-details}.}
\label{fig:qcqp_quality_latency}
\end{figure}

\FloatBarrier
\subsection{Block correction ablations}\label{sec:block-ablation-main}
Table~\ref{tab:block_comparison}'s operations share initial points, identity mobility, FP64, and $10^{-3}$ tolerance. Violation-GD and dual PGD match block sizes and iteration budgets. Block principal pivoting (BPP) allows $16$ pivots. Speedups over extrapolated minibatch stochastic subgradient projection (M-SSP)~\citep{necoara2019minibatch} use paired timings on the same GPU, as detailed in Appendix~\ref{app:block-ablation}.

\begin{table}[!htbp]
\caption{Median time to feasibility (ms), small/medium per cell, five successful instances each. Small/medium $(n,n_{\rm eq},m)=(256,32,2048)/(512,64,4096)$. }
\label{tab:synthetic_constraints}
\centering\setlength{\tabcolsep}{4pt}
\begin{tabular}{@{}lr@{\,/\,}rr@{\,/\,}rr@{\,/\,}r@{}}\toprule
Method & \multicolumn{2}{c}{Squared affine} & \multicolumn{2}{c}{Log-sum-exp} & \multicolumn{2}{c}{Rotated SOC}\\\midrule
Single row (frozen) & 29.745 & 107.956 & 34.383 & 100.015 & 18.444 & 66.106 \\
M-SSP & 16.916 & 41.322 & 23.284 & 67.219 & 11.842 & 73.798 \\
Violation-GD & 4.076 & 5.004 & 3.438 & 4.257 & 2.790 & 3.585 \\
BPP-16 & 54.299 & 43.175 & 33.791 & 28.204 & 29.333 & 24.068 \\
Dual PGD (ours) & \textbf{3.470} & \textbf{4.759} & \textbf{3.047} & \textbf{3.822} & \textbf{2.363} & \textbf{3.201} \\
\bottomrule\end{tabular}
\end{table}

Dual PGD is fastest across all six groups (Table~\ref{tab:synthetic_constraints}): $1.06$--$1.17\times$ over Violation-GD, $4.88$--$19.56\times$ over M-SSP, and $7.82$--$27.35\times$ over single-row repair with matched relinearization, using paired timings.

\section{Conclusion}
\method{} unifies convex constraint repair through residual and (sub)gradient oracles. GPU-friendly fixed-depth dual PGD gives outer convergence and linear convergence in distance under Slater's condition and bounded subgradients. Control, power-system, and public optimization tasks demonstrate low inference cost; mobility learned jointly with the predictor further improves objective gaps.

\clearpage
\subsection*{AI use statement}
In this work, we used generative AI tools to help refine theoretical models and mathematical arguments, assist with writing and reviewing mathematical proofs, assist with method implementation based on algorithms and pseudocode provided by the authors, and assist with translation into English. We did not use generative AI tools to generate experimental data, propose or refine research hypotheses, design or provide feedback on research methods or experiments, clean or reformat datasets, or interpret experimental results; the remaining tasks requiring disclosure are not applicable to this work. Additionally, we used generative AI tools to assist with writing and checking experimental code, literature search and synthesis, identifying relevant literature, and editing author-written text to improve readability. All AI-assisted work was reviewed by the authors. The authors take responsibility for the final content of this work, including text, claims, and other research materials produced with the assistance of generative AI.

\subsection*{Reproducibility statement}
Appendix~\ref{app:proofs} provides the full assumptions and proofs. Appendix~\ref{app:training-protocol} specifies data splits, training settings, stopping rules, mobility implementation, and complexity. Appendix~\ref{app:oracles} defines task constraints and oracles. Appendix~D provides all controls, fixed-depth sweeps, baseline failure records, and timing protocols.

\bibliography{references}
\bibliographystyle{iclr2027_conference}

\appendix
\section{Full Theoretical Analysis}
\label{app:proofs}

Fix an input $x$. We analyze Euclidean correction coordinates, applying the change of coordinates below first when mobility is fixed and positive definite. Assume $\Omega\ne\varnothing$, each $h_i:\mathbb R^q\to\mathbb R$ is finite and convex, and $a_i^k\in\partial h_i(z^k)$ is available. Finite convex functions on $\mathbb R^q$ are continuous, hence $\Omega$ is closed and convex. Bounds involving true violations also assume $\|a_i^k\|\le M<\infty$ for violated rows on the iterate region. The algorithm uses zero stopping tolerance and stays fixed after feasibility. Every local PGD solve starts at zero; its rows, signed residuals, and step size are fixed within that solve but may change between blocks. Norms below are Euclidean, except that $\|G\|_\infty$ denotes the maximum absolute row sum.

\subsection{Halfspace containment and block notation}
\label{app:mobility-theory}
\paragraph{Coordinate convention.}
For $P_z\succ0$ fixed during one repair, take its symmetric positive-definite square root and let $\xi=P_z^{-1/2}z\in\mathbb R^q$. Then $\widetilde h_i(\xi)=h_i(P_z^{1/2}\xi)$ remains convex, has subgradient $P_z^{1/2}a_i$, and satisfies
\begin{equation}
\operatorname{dist}(\xi,P_z^{-1/2}\Omega)
=\inf_{z^\star\in\Omega}\|z-z^\star\|_{P_z^{-1}}.
\end{equation}
The rectangular factor $R$ from Appendix~\ref{app:mobility-implementation} satisfies $RR^\top=P_z$. The transformed rows have the same pairwise inner products $a_i^\top P_z a_j$ under either representation, and hence identical row norms, normalized residual scores, and normalized Gram matrices. Given the same block choices and step sizes, zero-initialized finite dual PGD produces the same dual iterates and the same correction mapped back to $z$. Below we again call the transformed objects $z,h_i,\Omega$; subgradient and error-bound constants refer to these coordinates. Invertibility of $P_z^{1/2}$ makes feasible-point convergence equivalent in both coordinates.

\begin{lemma}[Nonzero subgradient at a violated row]
\label{lem:nonzero}
If $h_i(z)>0$, $a_i\in\partial h_i(z)$, and $\Omega\ne\varnothing$, then $a_i\ne0$.
\end{lemma}
\begin{proof}
If $a_i=0$, the subgradient inequality makes $z$ a global minimizer of $h_i$. Taking $z^\star\in\Omega$ would give $h_i(z)\le h_i(z^\star)\le0$, a contradiction.
\end{proof}

\begin{lemma}[Feasible-set containment]
\label{lem:containment}
For $a_i\in\partial h_i(z)$, define $H_i(z)=\{w:h_i(z)+\langle a_i,w-z\rangle\le0\}$. Then $\Omega\subseteq H_i(z)$, so any selected intersection also contains $\Omega$.
\end{lemma}
\begin{proof}
For $z^\star\in\Omega$, the subgradient inequality gives $h_i(z)+\langle a_i,z^\star-z\rangle\le h_i(z^\star)\le0$. Intersect over selected rows.
\end{proof}

For nonzero rows set $u_i=a_i/\|a_i\|$ and $s_i=h_i(z)/\|a_i\|$. A satisfied zero-subgradient row can be represented by $(u_i,s_i)=(0,0)$. Let $U$ stack the normalized rows of a block $B$ and let $s$ contain their signed residuals. Then
\begin{equation}
\mathcal C_B(z)=\{\Delta:U\Delta+s\le0\},\qquad
\Omega-z\subseteq\mathcal C_B(z).
\end{equation}
In particular this polyhedron is nonempty. Define
\begin{equation}
G_B=UU^\top,\quad \rho_B=\lambda_{\max}(G_B),\quad
\delta_B=[s]_+,\quad E_B=\|\delta_B\|^2,\quad
D_B=\|P_{\mathcal C_B(z)}(0)\|.
\end{equation}
The full Gram matrix includes satisfied filler rows with negative slack; the PGD spectral condition refers to this full matrix. For $E_B>0$, set $r_B=\delta_B^\top G_B\delta_B/E_B$. Every feasible $\Delta$ satisfies $\delta_B^\top U\Delta\le-\delta_B^\top s=-E_B$. Therefore $U^\top\delta_B\ne0$, and $0<r_B\le\rho_B$. The ratio is unused when $E_B=0$.

An implementation may use a norm floor $\nu_i=\max\{\|a_i\|,\varepsilon_{\rm norm}\}$ and normalize both row and residual by $\nu_i$. Containment is unchanged and row norms are at most one. The proofs then use those normalized violations and replace $M$ in residual bounds by $\max\{M,\varepsilon_{\rm norm}\}$.

\subsection{Dual progress for PGD alone}
\label{app:dual-proof}
The block projection and its dual are
\begin{equation}
\min_\Delta\tfrac12\|\Delta\|^2\quad\text{s.t. }U\Delta+s\le0,
\qquad
\max_{\lambda\ge0}\phi(\lambda)=s^\top\lambda-\tfrac12\lambda^\top G_B\lambda.
\end{equation}
Minimizing the Lagrangian over $\Delta$ gives $\Delta=-U^\top\lambda$.

\begin{lemma}[Dual value controls distance descent]
\label{lem:dual-descent}
For any $\lambda\ge0$, let $\Delta=-U^\top\lambda$ and $z^+=z+\eta\Delta$ with $0<\eta\le1$. Then, for every $z^\star\in\Omega$,
\begin{equation}
\|z^+-z^\star\|^2\le\|z-z^\star\|^2
-2\eta\phi(\lambda)-\eta(1-\eta)\|U^\top\lambda\|^2.
\end{equation}
\end{lemma}
\begin{proof}
Containment gives $U(z^\star-z)+s\le0$. Multiplication by $\lambda\ge0$ yields
\begin{equation}
\langle-U^\top\lambda,z-z^\star\rangle
=\lambda^\top U(z^\star-z)\le-s^\top\lambda.
\end{equation}
Expand the squared distance after the update and substitute the definition of $\phi$.
\end{proof}
The lemma also applies to correctly shifted frozen linearizations in Appendix~\ref{app:inner-composition}.

\begin{lemma}[Block progress at a stable PGD step size]
\label{lem:pgd-ascent}
For $G_B\ne0$, let $\lambda^0=0$ and $\lambda^{t+1}=[\lambda^t+\alpha(s-G_B\lambda^t)]_+$ with $0<\alpha<2/\rho_B$. For $q_t=\lambda^{t+1}-\lambda^t$,
\begin{equation}
\phi(\lambda^{t+1})-\phi(\lambda^t)
\ge(1/\alpha-\rho_B/2)\|q_t\|^2.
\end{equation}
If $E_B>0$, then for every $T\ge1$,
\begin{equation}
c_{B,T}:=2\phi(\lambda^T)
\ge\alpha(2-\alpha r_B)E_B
\ge\alpha(2-\alpha\rho_B)E_B>0.
\end{equation}
If $E_B=0$, all $\lambda^t$ are zero.
\end{lemma}
\begin{proof}
Set $F=-\phi$. Projection optimality with $\lambda^t$ as comparator gives $\langle\nabla F(\lambda^t),q_t\rangle\le-\|q_t\|^2/\alpha$. Since $F$ is quadratic with Hessian $G_B$,
\begin{equation}
F(\lambda^{t+1})-F(\lambda^t)
=\langle\nabla F(\lambda^t),q_t\rangle+\tfrac12q_t^\top G_Bq_t
\le-(1/\alpha-\rho_B/2)\|q_t\|^2.
\end{equation}
The first step is $\lambda^1=\alpha\delta_B$, and $s^\top\delta_B=E_B$, so $2\phi(\lambda^1)=2\alpha E_B-\alpha^2\delta_B^\top G_B\delta_B=\alpha(2-\alpha r_B)E_B$. Later dual values are nondecreasing. When $E_B=0$, $s\le0$ and zero remains fixed.
\end{proof}
Weak duality and containment imply $c_{B,T}\le D_B^2\le\operatorname{dist}(z,\Omega)^2$.

\subsection{Standard step size and finite dual depth}
\label{app:pgd-depth}
Assume $E_B>0$ and use $\alpha=1/L_B$ with $L_B\ge\rho_B>0$. Lemma~\ref{lem:pgd-ascent} immediately gives $c_{B,T}\ge E_B/L_B$; when $E_B=0$, return zero.

\paragraph{Finite-step comparison inequality.}
For every $\mu\ge0$, projection optimality and convexity give
\begin{equation}
F(\lambda^{t+1})-F(\mu)
\le\frac{L_B}{2}\bigl(\|\lambda^t-\mu\|^2-\|\lambda^{t+1}-\mu\|^2\bigr).
\end{equation}
Indeed, with $p=\lambda^{t+1}$ and $x=\lambda^t$, projection optimality gives
\begin{equation}
\langle\nabla F(x),p-\mu\rangle
\le\frac{L_B}{2}\bigl(\|x-\mu\|^2-\|p-\mu\|^2-\|p-x\|^2\bigr),
\end{equation}
while $F(p)-F(\mu)\le\langle\nabla F(x),p-\mu\rangle+\rho_B\|p-x\|^2/2$. Sum from $t=0$ to $T-1$. Since $\lambda^0=0$ and $F(\lambda^t)$ is nonincreasing,
\begin{equation}
\phi(\lambda^T)\ge\phi(\mu)-\frac{L_B}{2T}\|\mu\|^2.
\end{equation}

\begin{proof}[Proof of Theorem~\ref{thm:inexact-main}]
Take $\mu=a\delta_B$ with $a\ge0$ in the comparison inequality. Then
\begin{equation}
c_{B,T}\ge E_B\left[2a-(r_B+L_B/T)a^2\right].
\end{equation}
The right side is maximized at $a=1/(r_B+L_B/T)$. Combine the resulting bound with $c_{B,T}\ge E_B/L_B$, weak duality, and Lemma~\ref{lem:dual-descent} at $\eta=1$. If $E_B=0$, zero belongs to $\mathcal C_B(z)$, so $D_B=c_{B,T}=0$.
\end{proof}

\paragraph{Progress relative to exact block projection.}
For $E_B>0$, the minimum-norm displacement $\Delta_\star$ in the nonempty polyhedron is unique. Its normal cone is generated by active row normals, so a nonnegative optimal multiplier exists. Choose the minimum-Euclidean-norm one, $\lambda^\dagger$. KKT conditions give
\begin{equation}
\Delta_\star=-U^\top\lambda^\dagger,\qquad
s^\top\lambda^\dagger=\|U^\top\lambda^\dagger\|^2=D_B^2.
\end{equation}
Define $\chi_B=L_B\|\lambda^\dagger\|^2/D_B^2\ge1$. Applying the comparison bound with $\mu=a\lambda^\dagger$ gives
\begin{equation}
c_{B,T}\ge2aD_B^2-a^2\left(D_B^2+\frac{L_B}{T}\|\lambda^\dagger\|^2\right).
\end{equation}
Choosing $a=T/(T+\chi_B)$ yields
\begin{equation}
\frac{T}{T+\chi_B}D_B^2\le c_{B,T}\le D_B^2.
\end{equation}
For a fixed block the guaranteed fraction approaches exact projection as dual depth grows. A uniform fraction across blocks additionally requires a uniform upper bound on $\chi_B$.

\paragraph{Finite-depth displacement error.}
Let $\Delta_T=-U^\top\lambda^T$. Under the same step-size rule, the finite-step correction and the minimum displacement for the current halfspace block satisfy
\begin{equation}\label{eq:block-displacement-error}
\|\Delta_T-\Delta_\star\|^2
\le D_B^2-c_{B,T}
\le\frac{\chi_B}{T+\chi_B}D_B^2.
\end{equation}
\begin{proof}
Since $\lambda^T\ge0$ and $U\Delta_\star+s\le0$,
\begin{align}
\|\Delta_T-\Delta_\star\|^2
&=D_B^2-2\phi(\lambda^T)+2(\lambda^T)^\top(U\Delta_\star+s)\\
&\le D_B^2-c_{B,T}.
\end{align}
Substitute the preceding progress bound. If $E_B=0$, both displacements are zero. This bounds error relative to the current linearized block projection, not projection onto the original convex feasible set.
\end{proof}

\subsection{Step-size choice and amplified-step finite-depth bounds}
\label{app:gershgorin}
\paragraph{Gershgorin upper bound.}
For a nonzero Gram matrix, $\|G_B\|_\infty=\max_i\sum_j|(G_B)_{ij}|\ge\rho_B$. Choose $\ell_B\ge\|G_B\|_\infty$, possibly with numerical padding, and $\alpha=\omega/\ell_B$ for $0<\omega<2$. Recompute $\ell_B$ per block but hold the step fixed inside that solve; return zero for an all-zero block.

\begin{corollary}[Block descent with an amplified step]
\label{cor:gershgorin}
For any $T\ge1$ under this step-size rule,
\begin{equation}\label{eq:gershgorin-progress}
c_{B,T}\ge\frac{\omega(2-\omega)}{\ell_B}E_B.
\end{equation}
\end{corollary}
\begin{proof}
When $E_B>0$, Lemma~\ref{lem:pgd-ascent} and $\alpha\rho_B\le\omega$ yield $c_{B,T}\ge\alpha(2-\alpha\rho_B)E_B\ge\omega(2-\omega)E_B/\ell_B$. For $E_B=0$ both sides vanish.
\end{proof}
For $\omega=1.8$, the base guarantee is $c_{B,T}\ge0.36E_B/\ell_B$. We next retain dual-depth dependence, including $\alpha>1/\rho_B$.

\paragraph{Comparison inequality for a general stable step.}
For any $\mu\ge0$, projection optimality gives
\begin{equation}
F(\lambda^{t+1})-F(\mu)
\le\frac{\|\lambda^t-\mu\|^2-\|\lambda^{t+1}-\mu\|^2}{2\alpha}
+\frac{\rho_B-1/\alpha}{2}\|q_t\|^2.
\end{equation}
Let $A_\alpha=1/\alpha-\rho_B/2>0$ and $B_\alpha=(\rho_B-1/\alpha)_+/2$. Dual monotonicity and $\phi(\lambda^0)=0$ imply $\sum_{t<T}\|q_t\|^2\le\phi(\lambda^T)/A_\alpha$. Summing the comparison bound and dropping its nonpositive final squared-distance term yields
\begin{equation}
T\bigl(\phi(\mu)-\phi(\lambda^T)\bigr)
\le\frac{\|\mu\|^2}{2\alpha}
+\frac{B_\alpha}{A_\alpha}\phi(\lambda^T).
\end{equation}
With $\xi_\alpha=(\alpha\rho_B-1)_+/(2-\alpha\rho_B)$, this gives
\begin{equation}
\phi(\lambda^T)\ge\frac{T}{T+\xi_\alpha}
\left(\phi(\mu)-\frac{\|\mu\|^2}{2\alpha T}\right).
\end{equation}
The Gershgorin step satisfies $\alpha\rho_B\le\omega$. Because $\phi(\lambda^T)\ge0$, we may replace $\xi_\alpha$ by the larger $\xi_\omega=(\omega-1)_+/(2-\omega)$. Taking $\mu=a\delta_B$ and maximizing over $a\ge0$ gives
\begin{equation}
c_{B,T}\ge E_B\max\left\{
\frac{\omega(2-\omega)}{\ell_B},
\frac{T}{T+\xi_\omega}\frac1{r_B+\ell_B/(\omega T)}
\right\}.
\end{equation}
At $\omega=1.8$, $\xi_\omega=4$, and the two terms become $0.36/\ell_B$ and $[T/(T+4)]/[r_B+\ell_B/(1.8T)]$. Each local PGD block may use its own $\ell_B$. If $\omega$ varies across blocks, a uniform convergence constant requires it to stay uniformly away from both zero and two.

Likewise, setting $\mu=a\lambda^\dagger$ gives
\begin{equation}
c_{B,T}\ge\frac{T}{T+\xi_\omega}\frac{T}{T+\chi_{\omega,B}}D_B^2,
\qquad\chi_{\omega,B}=\frac{\ell_B\|\lambda^\dagger\|^2}{\omega D_B^2}.
\end{equation}
These bounds apply to the actual fixed-$T$ PGD output. For a positive Rayleigh estimate $\widehat L_B\le\rho_B$, the step $\alpha_B=1/\widehat L_B$ satisfies $\alpha_B<2/\rho_B$ exactly when $\widehat L_B>\rho_B/2$. Actual power and Gershgorin choices are recorded by experiment.

\subsection{Multiple block steps under a frozen linearization}
\label{app:inner-composition}
At outer point $z^k$, form all normalized rows $\bar U^k$ and signed residuals $\bar s^k$. Starting from $w^{k,0}=z^k$, use the correctly shifted affine residuals
\begin{equation}
s^{k,t}=\bar s^k+\bar U^k(w^{k,t}-z^k).
\end{equation}
Select block $B_{k,t}$, run zero-initialized fixed-depth PGD on it, and denote the dual output by $\lambda^{k,t}$. Set
\begin{align}
\Delta^{k,t}&=-(\bar U^k_{B_{k,t}})^\top\lambda^{k,t},\\
\phi_{k,t}(\lambda)&=(s^{k,t}_{B_{k,t}})^\top\lambda
-\tfrac12\|(\bar U^k_{B_{k,t}})^\top\lambda\|^2,\\
c_{k,t}&=2\phi_{k,t}(\lambda^{k,t}).
\end{align}

\begin{proposition}[Composition of inner corrections]
\label{prop:inner-composition}
Suppose each nonzero block uses $0<\alpha_{k,t}<2/\rho_{k,t}$, positive dual depth, and the full update $w^{k,t+1}=w^{k,t}+\Delta^{k,t}$; zero blocks return zero. Then $c_{k,t}\ge0$ and, for every $z^\star\in\Omega$,
\begin{equation}
\|w^{k,S}-z^\star\|^2\le\|z^k-z^\star\|^2-\sum_{t=0}^{S-1}c_{k,t}.
\end{equation}
\end{proposition}
\begin{proof}
For any frozen row and feasible point,
\begin{equation}
s_i^{k,t}+(\bar u_i^k)^\top(z^\star-w^{k,t})
=\bar s_i^k+(\bar u_i^k)^\top(z^\star-z^k)\le0.
\end{equation}
Thus every local block still contains the original feasible set. Lemma~\ref{lem:pgd-ascent} ensures $c_{k,t}\ge0$; apply Lemma~\ref{lem:dual-descent} at each inner step and sum.
\end{proof}
The first inner residuals come from the true nonlinear functions at $z^k$. Setting $c_k=c_{k,0}$, the complete outer step descends by at least $c_k$. Later steps use frozen affine residuals and add progress $C_k=\sum_{t<S}c_{k,t}$ with $c_k\le C_k\le d_k^2$.

With an additional external relaxation $0<\eta_{k,t}\le1$, the per-step decrement becomes $\eta_{k,t}c_{k,t}+\eta_{k,t}(1-\eta_{k,t})\|\Delta^{k,t}\|^2$. The main algorithm uses full steps. Uniform outer rates with relaxed steps also require a uniform positive lower bound on the first inner step's $\eta_{k,0}$. This relaxation scales the correction, whereas $\omega$ sets the inner PGD step size.

\subsection{Outer convergence and block-dependent rates}
\label{app:block-gain}
At each outer iterate set
\begin{equation}
\delta_i^k=\begin{cases}h_i(z^k)/\|a_i^k\|,&h_i(z^k)>0,\\0,&h_i(z^k)\le0,\end{cases}
\quad\Psi_k=\sum_i(\delta_i^k)^2,
\quad v_k=\max_i[h_i(z^k)]_+.
\end{equation}
Let the first block have $E_k=\tau_k\Psi_k$, and let any proved coefficient $\beta_k>0$ satisfy $c_k\ge\beta_kE_k$. Put $w_k=\tau_k\beta_k$.

Let $d_k=\operatorname{dist}(z^k,\Omega)$ and take $\beta_k=\max\{1/L_k,1/(r_k+L_k/T_k)\}$. The following uses Section~\ref{sec:theory}'s convexity, nonempty feasibility, and bounded-subgradient assumptions. Ratios refer to infeasible iterates; after feasibility, errors stay zero and the bounds hold trivially.

\begin{corollary}[Convex convergence and linear rate under a residual error bound]
\label{cor:outer-general}
Each local block starts PGD at zero, every nonzero block uses $\alpha_B=1/L_B$ with $\rho_B\le L_B\le\bar L<\infty$, and every correction is applied as a full step. Then Algorithm~\ref{alg:layer} satisfies
\begin{equation}
d_{k+1}^2\le d_k^2-w_k\Psi_k,\qquad
\sum_{k<K}w_k\Psi_k\le d_0^2,
\end{equation}
and the best-iterate bound
\begin{equation}
\min_{k<K}v_k\le\frac{Md_0}{\sqrt{\sum_{k<K}w_k}}.
\end{equation}
Fixed top-$b$ selection and these step sizes give $w_k\ge\underline w=b/(m\bar L)>0$. Thus the best-iterate bound is $O(K^{-1/2})$, and the full sequence converges to a point in $\Omega$. \emph{Linear rate.} If the constraints satisfy the normalized-residual error bound $d_k^2\le\kappa_2^2\Psi_k$ at the iterates for some $\kappa_2>0$, then
\begin{equation}
d_K^2\le d_0^2\prod_{k<K}\left(1-\frac{w_k}{\kappa_2^2}\right)
\le\left(1-\frac{\underline w}{\kappa_2^2}\right)^K d_0^2.
\end{equation}
\end{corollary}
\begin{proof}[Proof of Corollary~\ref{cor:outer-general}]
Proposition~\ref{prop:inner-composition} gives $d_{k+1}^2\le d_k^2-c_k\le d_k^2-w_k\Psi_k$. Telescoping yields $\sum_{k<K}w_k\Psi_k\le d_0^2$. The violated-row subgradient bound gives $v_k^2\le\sum_i[h_i(z^k)]_+^2\le M^2\Psi_k$. Therefore
\begin{equation}
\min_{k<K}v_k^2\sum_{k<K}w_k
\le\sum_{k<K}w_kv_k^2\le M^2d_0^2,
\end{equation}
proving the weighted best-iterate bound. If $w_k\ge\underline w>0$, then $\sum_k\Psi_k<\infty$ and $v_k\to0$. Pointwise distance descent makes the iterates Fej\'er monotone and bounded. A convergent subsequence has a feasible limit $\bar z$ by continuity of each $h_i$. The distance to $\bar z$ is nonincreasing and vanishes along that subsequence, so the full sequence converges to $\bar z$.

Under $d_k^2\le\kappa_2^2\Psi_k$, each infeasible iterate satisfies $d_{k+1}^2\le(1-w_k/\kappa_2^2)d_k^2$. Since $w_k\Psi_k\le c_k\le d_k^2\le\kappa_2^2\Psi_k$, the factors are nonnegative. Iterate to obtain the product and uniform linear bounds. If feasibility occurs in finitely many steps, later errors remain zero.
\end{proof}

\begin{lemma}[Slater condition implies a normalized-residual error bound]
\label{lem:slater-error-bound}
Suppose some $\bar z$ and $\sigma>0$ satisfy $h_i(\bar z)\le-\sigma$ for all $i$. For the outer iterates of Proposition~\ref{prop:inner-composition}, assume violated-row subgradient norms are bounded by $M$. With $R=\|z^0-\bar z\|$, every iterate satisfies
\begin{equation}
d_k\le\frac{R}{\sigma}v_k\le\frac{RM}{\sigma}\sqrt{\Psi_k}.
\end{equation}
Thus $\kappa_2=RM/\sigma>0$ is valid when $z^0$ is infeasible; if $z^0$ is feasible, all errors remain zero.
\end{lemma}
\begin{proof}
Proposition~\ref{prop:inner-composition} makes the outer iterates Fej\'er monotone with respect to $\bar z\in\Omega$, so $\|z^k-\bar z\|\le R$. If $v_k>0$, set $\theta=v_k/(\sigma+v_k)$ and $y=(1-\theta)z^k+\theta\bar z$. Convexity gives, for every $i$,
\begin{equation}
h_i(y)\le(1-\theta)h_i(z^k)+\theta h_i(\bar z)
\le(1-\theta)v_k-\theta\sigma=0.
\end{equation}
Hence $y\in\Omega$ and $d_k\le\|z^k-y\|=\theta\|z^k-\bar z\|\le Rv_k/\sigma$. For a violated row, $[h_i(z^k)]_+=\|a_i^k\|\delta_i^k\le M\delta_i^k$, so $v_k\le M\sqrt{\Psi_k}$. If $v_k=0$, then $d_k=\Psi_k=0$.
\end{proof}
With a norm floor $\varepsilon_{\rm norm}$, replace $M$ by $\max\{M,\varepsilon_{\rm norm}\}$.

Combining this lemma with Corollary~\ref{cor:outer-general} and $\underline w=b/(m\bar L)$ proves the main Slater rate in Corollary~\ref{cor:progress}.

\paragraph{Constants for fixed top-$b$.}
For $\alpha_B=\omega/\ell_B$, the same outer-convergence proof uses~\eqref{eq:gershgorin-progress} in place of $c_{B,T}\ge E_B/L_B$.
If $p_k$ normalized violations are positive, selecting the largest $b$ gives $\tau_k\ge\min\{1,b/p_k\}\ge b/m$. Unit row norms imply $|(G_B)_{ij}|\le1$, so $\ell_B=\|G_B\|_\infty\le b$. With fixed $\omega\in(0,2)$ and this Gershgorin choice, $w_k\ge\omega(2-\omega)/m$. If the implementation adds a norm floor or safety padding, use its actual uniform upper bound: with $\ell_k\le\bar\ell$, one has $w_k\ge\omega(2-\omega)b/(m\bar\ell)$.

\begin{corollary}[Block geometry and dual depth in a linear bound]
\label{cor:block-rate}
For $\alpha_k=1/L_k$, suppose $\tau_k\ge\underline\tau>0$, $r_k\le\bar r$, $\rho_k\le L_k\le\bar L$, $T_k\ge\underline T\ge1$, and the $\ell_2$ error bound of Corollary~\ref{cor:outer-general} holds. Set
\begin{equation}
\underline w=\underline\tau\max\left\{\frac1{\bar L},
\frac1{\bar r+\bar L/\underline T}\right\}.
\end{equation}
Then $d_K^2\le(1-\underline w/\kappa_2^2)^Kd_0^2$.
\end{corollary}
\begin{proof}
The lower bound in Appendix~\ref{app:pgd-depth} gives $w_k\ge\underline w$; apply Corollary~\ref{cor:outer-general}.
\end{proof}
For a fixed $\omega$ Gershgorin configuration with $\ell_k\le\bar\ell$, the analogous constant is
\begin{equation}
\underline w=\underline\tau\max\left\{
\frac{\omega(2-\omega)}{\bar\ell},
\frac{\underline T}{\underline T+\xi_\omega}
\frac1{\bar r+\bar\ell/(\omega\underline T)}\right\}.
\end{equation}

\paragraph{Fixed points and positive tolerances.}
Let $\mathcal T$ be one complete outer correction. At zero tolerance with fixed top-$b$, zero initialization, and safe step sizes, $\operatorname{Fix}(\mathcal T)=\Omega$: an infeasible point has $E_k>0$ in its first block and strictly decreases its distance to every feasible point, whereas all dual iterates remain zero at a feasible point. For positive tolerance, the best-iterate bound gives $O(M^2d_0^2/(\underline w\varepsilon^2))$ outer steps to the first point with $v_k\le\varepsilon$ if iterations continue. A finite outer cap instead returns its actual residual and iteration-limit status.

\subsection{Exact block projection and parallel correction}
\begin{lemma}[Descent for exact block projection]
\label{lem:exact-fejer}
Let $\Delta_\star=P_{\mathcal C_B(z)}(0)$ and $D_B=\|\Delta_\star\|$. For $0<\eta<2$ and $z^\star\in\Omega$,
\begin{equation}
\|z+\eta\Delta_\star-z^\star\|^2
\le\|z-z^\star\|^2-\eta(2-\eta)D_B^2,
\qquad D_B\le\operatorname{dist}(z,\Omega).
\end{equation}
\end{lemma}
\begin{proof}
Since $z^\star-z\in\mathcal C_B(z)$, projection optimality gives $\langle-\Delta_\star,z^\star-z-\Delta_\star\rangle\le0$, or $\langle\Delta_\star,z-z^\star\rangle\le-\|\Delta_\star\|^2$. Expand the squared distance. The displacement to a nearest feasible point is feasible for the block QP, proving $D_B\le\operatorname{dist}(z,\Omega)$.
\end{proof}

\begin{lemma}[Geometric lower bound for exact block projection]
\label{lem:gram-bound}
For $E_B>0$, $D_B^2\ge E_B/r_B\ge E_B/\rho_B$.
\end{lemma}
\begin{proof}
The inequality $\delta_B^\top U\Delta\le-E_B$ from Appendix~\ref{app:mobility-theory} and Cauchy--Schwarz imply $\|\Delta\|\ge E_B/\|U^\top\delta_B\|$ for every feasible displacement. Minimize over feasible displacements and square.
\end{proof}
The exact projection has optimal dual value $2\phi(\lambda^\dagger)=D_B^2$; Appendices~\ref{app:pgd-depth} and~\ref{app:gershgorin} quantify finite-step PGD progress relative to it.

\paragraph{Relation to extrapolated parallel subgradient projection.}
\label{app:ray-connection}
Let $J=\{i:s_i>0\}$, $p=|J|$, and average the violated-row closed-form corrections:
\begin{equation}
\bar\Delta=-p^{-1}U^\top\delta_B,\qquad
L_{\rm par}=\frac{\|\bar\Delta\|^2}{E_B/p}=\frac{r_B}{p}.
\end{equation}
With equal row weights and relaxation parameter one, the adaptive extrapolation of~\citet[Section 4.3]{necoara2019minibatch} gives
\begin{equation}
\frac{\bar\Delta}{L_{\rm par}}=-\frac1{r_B}U^\top\delta_B.
\end{equation}
This also maximizes the dual along the ray $\lambda=a\delta_B$ for $a\ge0$. The first PGD step is $-\alpha U^\top\delta_B=\alpha p\bar\Delta$; later Gram products adjust individual nonnegative row coefficients.

\subsection{Local rate for exact block corrections}
\begin{proposition}[Local quadratic decrease of infeasibility]
\label{prop:local-quadratic}
Let $z^\star\in\Omega$ and $I^\star=\{i:h_i(z^\star)=0\}$. Near $z^\star$, suppose active constraints are $C^2$ with $\|\nabla^2h_i(z)\|\le L_{\rm loc}$, inactive constraints have strict slack, the selected block includes $I^\star$, block projection is exact and uses a full step, and
\begin{equation}
\operatorname{dist}(z,\Omega)\le\kappa_{\rm loc}\max_i[h_i(z)]_+.
\end{equation}
Then sufficiently nearby iterates satisfy $d_{k+1}\le\kappa_{\rm loc}L_{\rm loc}d_k^2/2$.
\end{proposition}
\begin{proof}
For every active row, exact satisfaction of its linearized inequality and Taylor's theorem give $h_i(z^{k+1})\le L_{\rm loc}\|\Delta_\star^k\|^2/2$. Lemma~\ref{lem:exact-fejer} keeps the next point in a sufficiently small neighborhood and gives $\|\Delta_\star^k\|\le d_k$. Inactive constraints remain strictly satisfied there. Apply the local error bound at $z^{k+1}$.
\end{proof}

\section{Training and Implementation Details}\label{app:training-protocol}

\paragraph{Data sources and splits.}
Methods within each task use the same split, with sizes in Table~\ref{tab:experiment-splits}. The affine task uses the data and splits released with $\Pi$net. Convex QCQP uses SnareNet's released dataset and split with $100$ variables, $100$ quadratic inequalities, and $50$ equalities. The affine task's five seeds change training randomness and share one test set. Each ACOPF case has a fixed topology and a random split over load scenarios with split seed $0$.
\begin{table}[!htbp]
\centering\setlength{\tabcolsep}{3pt}
\caption{Training, validation, and test sizes. SOC-MPC data are generated separately for each horizon. Each ACOPF relaxation uses the listed sizes.}\label{tab:experiment-splits}
\begin{tabular}{lrrr}\toprule
Task & Training & Validation & Test\\\midrule
SOC-MPC, each $H$ & 8,000 & 1,000 & 1,000\\
Affine-constrained nonconvex optimization & 7,952 & 1,024 & 1,024\\
Convex QCQP & 8,334 & 833 & 833\\
ACOPF, 57 buses & 2,000 & 200 & 200\\
ACOPF, 118 buses & 4,000 & 400 & 400\\
ACOPF, 300 buses & 10,000 & 1,000 & 1,000\\\bottomrule
\end{tabular}
\end{table}

The synthetic ablation in Section~\ref{sec:block-ablation-main} trains no predictor and has no train/test split. Each of three constraint families and two sizes uses five problem-generation seeds, giving $30$ instances. Methods share instances and initial points; Appendices~C.1 and~\ref{app:block-ablation} describe generation and timing.

\paragraph{Finite unrolling.}
Let $L_{K_{\rm tr},S,T_{\rm dual}}(x,\hat y;P_\psi(x))$ denote the bounded correction operator. Each local block dual variable starts at zero and undergoes a fixed number of PGD steps. With tolerance masking, a sample remains unchanged after its original residual reaches tolerance. The corrected-output training objective is
\[
\min_{\theta,\psi}\frac1{|\mathcal D_{\rm tr}|}\sum_{(x,y^{\rm ref})\in\mathcal D_{\rm tr}}
\ell_{\rm task}\!\left(x,L_{K_{\rm tr},S,T_{\rm dual}}(x,f_\theta(x);P_\psi(x)),y^{\rm ref}\right).
\]
Raw-output and violation losses or mobility regularization can be added. For a fixed Euclidean geometry, set $P=I$ and omit $\psi$. An objective-only loss does not require reference labels. Training depth and inference cap are configured separately; the returned status records whether tolerance was reached.

We differentiate through the operations executed by the finite unroll, holding top-$b$ indices and Boolean branches fixed locally. Differentiating $a_i=N^\top\nabla_y g_i$ involves constraint Hessian information; normalization, the Gram matrix, and its step-size bound also belong to the computation graph. Detaching $a_i$ instead gives a surrogate gradient. Mobility is fixed within one forward repair but differentiated through all its repeated uses. For SDP, the eigenvector defining a minimum-eigenvalue subgradient is detached in the backward pass. The forward row remains a valid selected subgradient, while the backward pass treats this eigenvector as constant and uses a surrogate gradient.

\paragraph{Step sizes.}
ACOPF training uses the Gershgorin step $1.95/\ell_B$, and the other training configurations use $1.8/\ell_B$, where $\ell_B\ge\|G_B\|_\infty$. ACOPF and QCQP inference use a one-step power estimate, while SOC-MPC retains the Gershgorin step. The power estimate may not satisfy the step-size condition in Appendix~\ref{app:gershgorin}. With mobility, $G_B$ is formed from the transformed normalized rows.

\paragraph{SOC-MPC data generation.}
The six states are horizontal position and velocity, vertical position and velocity, pitch, and angular velocity. Dynamics use forward Euler discretization with step $0.05$. Positions and pitch integrate their respective velocities; horizontal acceleration is $9.81$ times pitch, vertical acceleration is $9.81(u_1+u_2)/2$, and angular acceleration is $6(u_2-u_1)$. Cost and constraint parameters are
\[
\begin{gathered}
Q=\operatorname{diag}(8,0.25,10,0.25,0.6,0.12),\qquad R=0.08I_2,\qquad Q_H=5Q,\\
D=\operatorname{diag}(1,0.2,1,0.2,0.35,0.12),\qquad W_u=I_2,\qquad
\rho=0.11,\quad u_{\max}=0.45.
\end{gathered}
\]
Initial-state components are sampled independently and uniformly: pitch from $[-0.04,0.04]$ and the other five components from $[-0.08,0.08]$. Horizontal and vertical targets are sampled from $[-0.38,0.38]$ and $[0.10,0.48]$, respectively. Seed $42+H$ generates independent training, validation, and test inputs sequentially at each horizon, all from the same distribution.

Tube centers are constructed during data generation. We first solve the quadratic objective subject only to the dynamics equalities, adding $10^{-10}I$ to its Hessian. At time $t$, we perturb the resulting control by
\[
0.055\tanh(2x_{\rm tar}^{\rm horiz})\sin\!\left(\frac{\pi t}{H-1}\right)(1,-1)^\top,
\qquad 0\le t<H,
\]
where $x_{\rm tar}^{\rm horiz}$ is the horizontal target position. If the largest control norm exceeds $0.72u_{\max}$, we scale the entire control sequence by a common factor to this bound. Rolling out the dynamics from the initial state gives tube centers $r_0,\ldots,r_H$. This produces a trajectory satisfying every inequality strictly; Appendix~C.2 gives the SOC constraints.

\paragraph{SOC-MPC training.}
For Section~\ref{sec:soc-mpc}, ours and $\Pi$net train for 100 epochs and use the final checkpoint. All horizons use 8,000 training and 1,000 validation instances, two 200-unit ReLU hidden layers without batch normalization, seed 0, Adam with learning rate $3\times10^{-4}$, and batch size 512. Predictor initial weights match; the mobility head starts at zero. Validation selects its learning rate: $3\times10^{-5}$ for $H=20$ and $3\times10^{-4}$ for $H=40,80$, while the predictor rate stays at $3\times10^{-4}$. The lower H20 rate scales the mobility-head update by $0.1$ after Adam computes it, leaving its moment states unchanged. Parameters and Adam states use FP32 and forward computation uses FP64. Training minimizes the batch mean of $\log(f_x(y_{\rm corr})+10^{-8})$, without solution labels. Ours adds a $10^{-5}$ squared penalty on centered $\log p$ with bound $\log 4$. TorchAX/JAX differentiates through the bounded repair. Ours uses training outer cap 32 and block capacity 64, giving effective block size $b=\min\{64,2H\}$. Each linearization uses five inner corrections and 20 dual steps per block, with tolerance $10^{-3}$ and per-sample masking. Training, validation, and test use the Gershgorin step $1.8/\|UU^\top\|_\infty$; validation and test use outer cap 250. Each horizon has 1,000 independent test instances. Identity mobility and the H20 common-rate sensitivity control appear in Appendix~\ref{app:soc-mpc-controls}.

$\Pi$net uses its official fixed-depth projection and custom VJP. Its training forward depths are $(200,300,400)$ for $H=(20,40,80)$, backward depth 200, $\omega=1.7$, and $\sigma=(0.0121563,0.0132498,0.00281115)$. The main table selects $K=400$ at all three horizons as a representative configuration from the fixed-depth test sweep. It is the smallest evaluated depth making every test output feasible at $10^{-3}$ at each horizon. The original validation-selected depths are $(200,300,400)$; selecting the main-table representative points uses test-sweep results. Every instance runs all specified steps without residual checks or early stopping. Appendix~\ref{app:soc-mpc-controls} retains the validation-selected configurations, original depths $(1000,1000,400)$, and full sweep on the same final checkpoints. Neural methods are timed in separate synchronous batch-one runs on RTX 4090 GPUs over the same ordered 1,000 instances; the physical GPU IDs are recorded with each run. Timing includes predictor, correction, and objective evaluation, excluding input transfer, compilation, warm-up, and the independent CPU audit. Objective and residuals are checked with the original NumPy problem definition.

\paragraph{SOC-MPC SnareNet.}\label{app:soc-mpc-snarenet}
At $H=20,40$, SnareNet uses the main experiment's predictor, initialization, sample order, $100$ epochs, Adam learning rate $3\times10^{-4}$, and batch size $512$. Training minimizes only the log corrected objective, without a violation penalty. Parameters and optimizer states use FP32; forward repair uses FP64. The official LM update handles dynamics equalities and SOC inequalities together, with regularization $0.01$, training threshold $10^{-3}$, and cap $100$. It stops on either the residual or update threshold. AdaRel starts from the per-constraint maximum training violation and decays linearly over the first $50$ epochs.

At $H=80$, batch-$512$ training encountered OOM with the default XLA (BFC) allocator on a 24\,GB RTX 4090.

At $H=20,40$, testing reuses the final checkpoint and selects only the internal threshold and iteration cap on validation data; external feasibility always uses $10^{-3}$ original-residual tolerance. Appendix~\ref{app:soc-mpc-controls} reports calibration and test results. The implementation retains the official LM update, adapting control flow and diagonal regularization for compilation. Forward and gradient computations were checked against the official implementation. Timing includes prediction, repair, and objective evaluation on all $1{,}000$ single-instance inputs, excluding compilation, warm-up, and transfer.

\paragraph{SOC-MPC DC3 and solver controls.}
DC3 predicts $2H$ controls, completes states via linear dynamics, and corrects squared positive inequality violations with momentum 0.5. It trains for 100 epochs with the same network, data, seed, optimizer rate, and batch size. The loss adds an inequality-violation norm penalty to log objective. Its training correction has fixed depth; inference stops at tolerance $10^{-3}$. Validation selects the following settings before test.
\begin{table}[t]
\centering\setlength{\tabcolsep}{3pt}
\caption{SOC-MPC DC3 correction settings.}\label{tab:soc_mpc_dc3_config}
\begin{tabular}{rrrrr}\toprule
$H$ & Step & Train steps & Penalty & Test cap\\\midrule
20 & $0.1$ & 5 & 0.5 & 500\\
40 & $8.75876\times10^{-4}$ & 20 & 5 & 8000\\
80 & $3.54856\times10^{-6}$ & 5 & 5 & 12800\\\bottomrule
\end{tabular}
\end{table}
H40 and H80 steps equal 1.5 times the inverse spectral upper bound for the squared-violation gradient. All test cases, including six H80 failures (maximum violation $2.7283\times10^{-3}$), enter mean gap and time.

Gurobi 13.0.1 solves the convex QCP; MOSEK 11.2.0 uses SOC constraints and a rotated-cone epigraph for the quadratic objective. Both are single-threaded on AMD EPYC 9654, reuse models without warm starts, and set only primal feasibility tolerance to $10^{-3}$; other tolerances remain default. All 1,000 instances per horizon pass the original-constraint check. Their times include parameter update, optimization, solution extraction, and objective evaluation, excluding initial setup and warm-up. SOC-MPC gaps retain their signs.

\paragraph{Training on the affine-constraint task.}\label{app:affine-protocol}
In Section~\ref{sec:affine}, $\Pi$net, SnareNet, JAXopt, and our method use predictors with two $200$-unit hidden layers, Adam learning rate $0.01$, batch size $2{,}048$, and $50$ training epochs, with JAX/XLA execution. JAXopt calls OSQP at tolerance $10^{-3}$ to compute the Euclidean projection of the raw prediction onto the feasible set of the affine constraints and trains the predictor through this layer.

DC3 uses its official PyTorch network and training configuration: two $200$-unit hidden layers with batch normalization and dropout, Adam learning rate $10^{-4}$, batch size $200$, and $1{,}000$ epochs. Equality completion precedes partial correction, with both training and test capped at ten correction steps, internal tolerance $10^{-4}$, step size $10^{-7}$, and momentum $0.5$. Equality and inequality soft penalties each have weight $5$. All methods use the same data split and final checkpoints. Reported latency corresponds to each implementation; external feasibility always uses original residuals at tolerance $10^{-3}$.

\paragraph{Convex QCQP training and inference.}\label{app:qcqp-training}
Section~\ref{sec:qcqp} uses SnareNet's public dataset with $100$ variables, $100$ convex quadratic inequalities, and $50$ affine equalities. All four joint configurations and the shared post predictor use the same split, seed $456$, two $200$-unit hidden layers, batch normalization, dropout $0.2$, FP64, Adam with learning rate $10^{-4}$, and batch size $100$. Each trains for $1{,}500$ epochs and uses its final checkpoint. The loss is
\[
\mathcal L(x,y)=f_x(y)+10\left(\|Cy-x\|_2^2+\|[g(y)]_+\|_2^2\right).
\]
Joint training evaluates this loss on corrected outputs. The post predictor uses raw outputs and no training-time repair. Reference objective values do not enter the training loss.

DC3 uses the QCQP baseline configuration released by SnareNet. It predicts $50$ free variables and completes the other $50$ through the equalities. Training unrolls ten partial-correction steps with step size $10^{-12}$ and momentum $0.5$, evaluating the loss on original constraints. The main-table run extends the released $1{,}000$-epoch budget to $1{,}500$ epochs to match the other methods; Appendix~\ref{app:qcqp-details} retains both budgets.

DC3's extra test correction is calibrated only after training. The first ten steps retain their training step size and momentum. On all $833$ validation inputs, we search extra step sizes $\{10^{-12},10^{-11},\ldots,10^{-4}\}$ with cap $10{,}000$ and stopping tolerance $10^{-6}$. We reject nonfinite candidates and those that fail this tolerance on any validation input. Among candidates within $0.01$ percentage points of the best mean Gap, we select the fewest mean extra iterations, giving step size $10^{-5}$. The extra phase retains DC3's momentum reset and mandatory first step. Settings are fixed before testing; external feasibility still uses $10^{-4}$.

SnareNet retains its official LM repair with regularization $0.01$, cap $100$ in training and testing, and internal threshold $10^{-6}$. It stops on either the residual or update threshold. SnareNet uses its AdaRel strategy, linearly reducing constraint slack to zero over the first $600$ epochs. All validation and test evaluations use the original constraints.

\method{} trains with the original constraints throughout. It uses $b=32$, five block updates per outer iteration, and ten dual PGD steps per block. Training has outer cap two and the Gershgorin step with scale $1.8$; testing has outer cap $100$, internal tolerance $10^{-6}$, and power1. The joint mobility head starts at zero; post uses identity mobility. The distinction between the empirical power1 setting and theoretically safe steps is as described above.

$\Pi$net uses the official Douglas--Rachford updates with $\omega=1.7$ and implicit BiCGSTAB differentiation. Each $g_i(y)=y^\top H_i y+a_i^\top y-h_i\le0$ is represented by the equivalent SOC
\[
\left\|\begin{bmatrix}2H_i^{1/2}y\\h_i-a_i^\top y-1\end{bmatrix}\right\|_2
\le h_i-a_i^\top y+1.
\]
The dataset has nonnegative diagonal $H_i$. We vectorize the SOC projections and evaluate the lifted affine projection using an equivalent $100$-dimensional Schur formula with cached coefficients. The implementation retains the official auxiliary variables, zero initialization, and distance objective on the original variables, without constructing a dense $10{,}300$-dimensional projector. Forward and implicit-gradient computations were checked against the official implementation on an actual-data subset of three constraints; QCQP/SOC equivalence was checked for all $100$ constraints.

$\Pi$net training uses a fixed $200$ forward steps and an implicit-backward cap of $200$, without training-time early stopping. Training $\sigma$ is selected from nine logarithmically spaced candidates between $0.001$ and $5.05$, using the initial predictor's first $128$ validation inputs. Among candidates with maximum original residual at most $0.05$, selection minimizes mean correction distance; if none qualifies, it minimizes maximum residual. This selects $5.05$. After training, joint and post separately calibrate test $\sigma$ and depth on all $833$ validation inputs, with depths $\{50,100,200,400,800,1600,3200,6400\}$. Candidates must satisfy $10^{-6}$ original-residual tolerance and have mean correction distance within a factor $1.01$ of the smallest feasible candidate. Selection then minimizes depth, breaking ties by residual and distance. Joint selects $(\sigma,K)=(0.5990575329,200)$; post selects $(1.7393218624,400)$. Main-text curves hold these $\sigma$ values and checkpoints fixed and change only inference depth.

This test calibration follows $\Pi$net's principle of screening for both feasibility and projection quality~\citep[Appendix~C.2]{grontas2026pinet}. We use the smallest mean correction distance among feasible candidates as an empirical reference for projection quality, with the candidate grid and $1\%$ threshold specified above.

All methods use FP64, batch-one timing on RTX 4090 GPUs, including prediction and repair but excluding compilation, warm-up, input transfer, and independent metric evaluation. Each point uses three complete test passes. The four SnareNet and \method{} configurations run sequentially on one physical GPU; $\Pi$net and DC3 are timed separately on another RTX 4090. SnareNet and our saved outputs agree with native Torch to maximum absolute error below $3\times10^{-15}$. At the main-table depths, $\Pi$net's sweep outputs also match its previously saved predictions. DC3 implements the same forward, loss, and gradient computations in JAX/XLA and is checked against the released Torch code; the first, middle, and last test instances in each configuration pass native-output checks. Appendix~\ref{app:qcqp-details} gives additional statistics.

\paragraph{ACOPF scenarios and reference solutions.}
The six datasets use the IEEE 57-, 118-, and 300-bus cases from PGLib. Load scenarios use GridFM DataKit's ERCOT 2024 hourly aggregate load series. We interpolate the series to the candidate count and linearly rescale it to a case-specific global load multiplier range. Each bus's active load receives an additional independent uniform multiplier; reactive loads stay at their original case values. The global ranges are $[0.8,1.05]$, $[0.7,1.25]$, and $[0.8,1.05]$ for 57, 118, and 300 buses. Corresponding bus-level multiplier ranges are $[0.9,1.1]$, $[0.8,1.2]$, and $[0.95,1.05]$. Generation uses seed $0$.

Each dataset generates three times the required number of candidate scenarios. PowerModels, JuMP, and MOSEK solve the corresponding relaxation and provide stored reference outputs and objectives. We retain only solved candidates whose original equality and inequality residuals are at most $10^{-3}$. We then select evenly spaced positions in the ordered retained candidates and apply the random split in Table~\ref{tab:experiment-splits}. Candidate filtering is performed separately for SOCP and SDP. Pretraining, fine-tuning, and testing use the same version of data, labels, and splits.

\paragraph{ACOPF training and inference.}\label{app:acopf-protocol}
Pretraining minimizes raw-prediction task MSE against reference outputs on each case's training split, without a constraint layer in training. It uses AdamW with peak learning rate $0.05$, weight decay $10^{-4}$, batch size $512$, gradient-norm clipping at $100$, and seed $0$. The rate warms up linearly from zero over $50$ epochs, then follows a cosine schedule of total length $2{,}000$ epochs. We save the shared pretrained checkpoint at epoch $1{,}900$. Predictor architecture and output processing for task MSE match subsequent fine-tuning.

Each of the six cases starts from the same case-specific 1,900-epoch pretrained predictor. Ours (post, full FT) fine-tunes the full predictor for 100 epochs, while Ours (post, head FT) freezes the trunk and running statistics and fine-tunes only the solution head. Both use raw-output task MSE without repair during training and apply $P=I$ only at inference. Ours (joint) freezes the pretrained trunk and running statistics and trains the solution and mobility heads for 100 epochs with
\[
\mathcal L=\tfrac12\mathcal L_{\rm raw}+\tfrac12\mathcal L_{\rm corr}
+10^{-5}\mathbb E_x\!\left[n^{-1}\|\log p_\psi(x)\|^2\right].
\]
$\mathcal L_{\rm raw}$ and $\mathcal L_{\rm corr}$ are task MSEs against dataset labels with the same output processing and column scaling. The mobility head starts at the identity. Both heads have peak learning rate 0.005 with ten warm-up epochs and cosine decay; existing AdamW state is retained for existing parameters. Predictors have two 200-unit hidden layers, batch normalization, and dropout 0.2. Effective batch size is 512, gradient clipping threshold 100, and seed 0. We report one final checkpoint per configuration, not cross-seed means.

Training uses two outer iterations, block size 256, and the $1.95$ Gershgorin step. Test uses TorchAX/XLA, an FP32 predictor and FP64 correction, block size 256, outer cap 5,000, and tolerance $10^{-3}$, with one-power-iteration step estimates. The budgets below apply to both relaxations.
\begin{table}[t]
\centering\setlength{\tabcolsep}{3pt}
\caption{ACOPF train and test correction budgets.}\label{tab:acopf-budgets}
\begin{tabular}{rrrrrr}\toprule
Buses & Train outer & Train inner & Train dual & Test inner & Test dual\\\midrule
57 & 2 & 4 & 4 & 4 & 8\\
118 & 2 & 10 & 10 & 10 & 20\\
300 & 2 & 20 & 20 & 20 & 40\\\bottomrule
\end{tabular}
\end{table}
Synchronous batch-one RTX 4090 inference times include the predictor, equality projection, and inequality correction, averaged over three trials. Original residuals are audited outside timing.

For ACOPF, a negative per-instance objective difference can occur within the feasibility tolerance. Tables report the mean of per-instance truncated relative gaps,
\[
\operatorname{Gap}(\%)=M^{-1}\sum_{j=1}^M[\zeta_j]_+,
\quad \zeta_j=100\frac{f_{x_j}(y_j)-f_j^{\rm ref}}{\max\{|f_j^{\rm ref}|,\epsilon_{\rm mach}\}}.
\]
All SOCP and SDP comparisons use this convention.

\paragraph{ACOPF layer baselines.}
For SOCP, official $\Pi$net 0.2.0 receives the same Ours (post, full FT) predictor output and uses its affine and SOC interfaces; thermal limits are represented equivalently as SOCs. Its original-state iteration stops at first $10^{-3}$ original-residual feasibility or 5,000 iterations. On 16 validation instances, $\sigma\in\{0.01,0.1,1,5\}$ is selected by feasible count then iteration count with $\omega=1.7$; all sizes select $\sigma=5$. It succeeds on $200/200$, $394/400$, and $992/1000$ instances for 57, 118, and 300 buses. Gap and time include failures. $\Pi$net 0.2.0 lacks a PSD-cone interface, so we omit it for SDP.

Official FSNet feasibility search also starts from the same Ours (post, full FT) checkpoint, uses memory $30$, scale $1000$, at most ten line-search steps, Armijo coefficient $10^{-4}$, and backtracking factor $0.5$. Objective and gradient stopping thresholds are $10^{-9}$ and $10^{-6}$, respectively, with an outer cap of 5,000 iterations. Original residuals are independently checked at $10^{-3}$. It fails on 12/400 SOCP118 and 915/1000 SOCP300 tests; full-set means are reported.

For SDP57, 118, and 300, LMI-Net implements the affine-set/PSD-cone Douglas--Rachford updates in Algorithm~1 and Eq.~(15) of~\citet{tang2026lminet}, using the Ours (post, full FT) checkpoint. We first project the raw prediction onto the equality set to obtain an anchor. Reduced coordinates start at zero, and the matrix variable starts at the anchor's LMI image. On the first $16$ validation instances, $\sigma\in\{0.01,0.1,1,5\}$ is selected by the highest feasible count, then the lowest mean iteration count; all three cases select $\sigma=0.01$. The cap is $5{,}000$ iterations, and stopping and evaluation use original residuals at tolerance $10^{-3}$. Both baselines' times include prediction and repair, exclude setup and compilation, and use synchronized batch-one evaluation.

\paragraph{ACOPF SnareNet postprocessing.}\label{app:acopf-snarenet}
SnareNet~\citep{chu2026snarenet} uses the same predictor checkpoint as Ours (post, full FT) in each evaluated case, without retraining. Starting from the raw network output, its official dense LM update jointly corrects equalities and inequalities, with regularization $0.01$, zero constraint slack, and an iteration cap of $5{,}000$. It stops when original residuals reach $10^{-3}$ or the official residual or update criterion reaches $10^{-8}$. Final feasibility always uses the original residuals at $10^{-3}$.

The predictor uses FP32 and repair uses FP64. The complete predictor and repair are compiled into one JAX/XLA executable. RTX 4090 timing uses synchronized batch-one inference, two warm-ups, and three full test passes; it includes prediction and repair but excludes compilation, input transfer, and independent metric evaluation. The completed SOCP57, SOCP118, SDP57, and SDP118 runs cover all $200$, $400$, $200$, and $400$ test instances, respectively, with $200$, $353$, $175$, and $105$ feasible outputs. Gap and latency include infeasible outputs.

\paragraph{Numerical status.}
The layer reports success only when the true residual reaches tolerance with finite output and oracle values; otherwise it reports the iteration limit. Satisfied zero-gradient rows can be masked as padding. A block with no positive residual produces no movement from zero-initialized dual PGD.

\subsection{Mobility parameterization and implementation}\label{app:mobility-implementation}
Euclidean block repair minimizes displacement equally in all directions, although the task objective may respond differently across directions. We introduce a learnable mobility that changes directional costs without changing the constraint set.

\paragraph{Definition and parameterization.}
Let $P_\psi(x)=\operatorname{diag}(p_1(x),\ldots,p_n(x))$ with $p_j(x)>0$. The learnable parameters are $\psi$; the matrix stays fixed during one complete repair for a fixed input. Suppressing $x,\psi$ below, its action on reduced directions is
\begin{equation}
P_z=N^\top PN\succ0.
\end{equation}
The unconstrained parameters $t_\psi(x)\in\mathbb R^n$ can be learned directly as a vector shared across instances or produced by an input-conditioned model. Our learned-mobility experiments use the latter. One bounded differentiable parameterization uses a fixed $\ell>0$:
\begin{equation}
v_j=\ell\tanh(t_{\psi,j}(x)),\qquad
\bar v=n^{-1}\sum_jv_j,\qquad p_j=\exp(v_j-\bar v).
\end{equation}
Then $\prod_jp_j=1$, eliminating a uniform scale degree of freedom, and $\max_jp_j/\min_jp_j\le e^{2\ell}$. Zero-output initialization gives $P=I_n$ and $P_z=I_q$.

\paragraph{From correction direction to weighted cost.}
Let $A_B\in\mathbb R^{b\times q}$ contain selected, unnormalized subgradient rows $a_i^\top$, and let $s_B^{\rm raw}=(\ell_i^k(w))_{i\in B}$ be their signed affine residuals at the current inner point. The linearized block constraints are $A_B\Delta z+s_B^{\rm raw}\le0$. The Euclidean QP has stationarity relation $\Delta z=-A_B^\top\lambda$ with $\lambda\ge0$. Mobility changes the direction of this nonnegative gradient combination to $\Delta z=-P_zA_B^\top\lambda$, or $P_z^{-1}\Delta z+A_B^\top\lambda=0$. This is the stationarity relation of the weighted QP
\begin{equation}
\label{eq:mobility_primal_appendix}
\min_{\Delta z}\tfrac12\Delta z^\top P_z^{-1}\Delta z
\quad\text{s.t.}\quad A_B\Delta z+s_B^{\rm raw}\le0.
\end{equation}
The metric $P_z^{-1}$ makes equal-length moves less costly along large-eigenvalue directions of $P_z$. At $P_z=I_q$, this reduces to Euclidean repair. For one violated affine row, its exact displacement is
\begin{equation}
\Delta z_i=-\frac{[\ell_i^k(w)]_+}{a_i^\top P_za_i}P_za_i.
\end{equation}
Accordingly, weighted selection scores $[\ell_i^k(w)]_+/\sqrt{a_i^\top P_za_i}$ and still uses a fixed-capacity top-$b$ block.

\paragraph{Reusing the Euclidean solver.}
Since $P_z=N^\top PN$, take the rectangular factor $R=N^\top P^{1/2}\in\mathbb R^{q\times n}$, for which $RR^\top=P_z$. Write $\Delta z=R\Delta\nu$ with $\Delta\nu\in\mathbb R^n$. For a given $\Delta z$, its minimum Euclidean representation cost is
\begin{equation}
\min_{R\Delta\nu=\Delta z}\tfrac12\|\Delta\nu\|^2
=\tfrac12\Delta z^\top(RR^\top)^{-1}\Delta z
=\tfrac12\Delta z^\top P_z^{-1}\Delta z.
\end{equation}
Therefore~\eqref{eq:mobility_primal} is equivalent to
\begin{equation}
\min_{\Delta\nu}\tfrac12\|\Delta\nu\|^2
\quad\text{s.t.}\quad A_BR\Delta\nu+s_B^{\rm raw}\le0.
\end{equation}
This is the Euclidean block problem of Section~\ref{sec:single-block} with each row transformed to $R^\top a_i=P^{1/2}Na_i$. Its squared norm is $a_i^\top P_za_i$, giving the weighted selection score above. We normalize and select these rows, run the same dual PGD, and recover $\Delta z=R\Delta\nu$ and $\Delta y=N\Delta z$.

\subsection{Computational complexity}\label{app:complexity}
Given constraint residuals and (sub)gradients, the inequality-block correction requires neither matrix factorization nor linear-system solves. Dual PGD uses matrix products, reductions, and elementwise operations: each block forms $G_B=UU^\top$ once, reuses it for $T_{\rm dual}$ steps, and maintains $\lambda\ge0$ by elementwise clipping. Our step-size rules likewise use only reductions or a fixed number of matrix--vector products. A fixed block size gives samples the same block tensor shapes, enabling batched execution and GPU parallelism.

We analyze the forward cost for one sample with dense coefficients. For $U\in\mathbb R^{b\times q}$, forming $G_B$ costs $O(b^2q)$, dual PGD costs $O(T_{\rm dual}b^2)$, and computing $-U^\top\lambda$ costs $O(bq)$. The block workspace stores $O(bq+b^2)$ values; retaining the current linearization requires another $O(mq+m)$.

Each outer iteration reuses one linearization for $S$ inner block updates. Let $\mathcal C_{\rm lin}$ denote the cost of true-residual queries, (sub)gradient evaluation, coordinate transforms, and normalization, and let $\mathcal C_{\rm sel}(m,b)$ denote the cost of selecting top-$b$ from $m$ scores. Updating affine residuals costs $O(mq)$ per inner step, giving an outer-iteration cost of
\begin{equation}\label{eq:outer_cost}
O\!\left(\mathcal C_{\rm lin}
+S\left[mq+\mathcal C_{\rm sel}(m,b)+b^2q+T_{\rm dual}b^2\right]\right).
\end{equation}
The total sums these costs over the executed outer iterations and adds equality preprocessing and the final residual check. Oracle costs depend on the constraint family; for example, the minimum-eigenpair computation for SDP belongs to $\mathcal C_{\rm lin}$.

Learned mobility preserves this computational structure. Let $J$ contain the original-coordinate constraint-gradient rows. The transformed gradients satisfy $JNR=(JNN^\top)P^{1/2}$. Since $P$ is positive diagonal, column scaling of equality-projected gradients implements this transform without forming $P_z^{-1}$ or taking a Cholesky factorization of $P_z$. When inner corrections accumulate in the $n$-dimensional transformed coordinates, the block-operation and linearization-storage costs above replace $q$ by $n$; the coordinate transforms at each outer iteration belong to $\mathcal C_{\rm lin}$.

\section{Constraint Residuals and Task Definitions}\label{app:oracles}

\subsection{Convex residual primitives}
All inequalities have the form $g_i(y,x)\le0$, with affine equalities. Coefficients are fixed for a given instance. Multiplying a selected $y$-space subgradient by $N^\top$ gives its reduced-coordinate row. Constraint families change the corresponding value and (sub)gradient queries; block selection and dual PGD use no projector onto their full feasible sets. The PSD row still requires a minimum-eigenpair computation.
\begin{table}[t]
\centering\setlength{\tabcolsep}{3pt}
\caption{Decision-space residuals and selected subgradients.}\label{tab:residual-primitives}
\begin{tabular}{lll}\toprule
Primitive & $g(y)$ & Subgradient\\\midrule
Affine & $a^\top y+c$ & $a$\\
Squared affine & $\frac12(a^\top y+c)^2-b$ & $(a^\top y+c)a$\\
Log-sum-exp & $\log\sum_j e^{p_j^\top y+\beta_j}-b$ & $\sum_j\pi_jp_j$\\
Norm minus affine & $\|My+c\|-t^\top y-b$ & $M^\top u-t$\\
Squared norm & $\|Fy+c\|^2-b$ & $2F^\top(Fy+c)$\\
PSD & $-\lambda_{\min}(W(y))$ & $(-v^*W_jv)_j$\\\bottomrule
\end{tabular}
\end{table}
Here $\pi_j=e^{p_j^\top y+\beta_j}/\sum_l e^{p_l^\top y+\beta_l}$. In the norm row, $u=(My+c)/\|My+c\|$ away from zero, and $u=0$ at zero. For PSD, $W(y)=W_0+\sum_j y_jW_j$ is Hermitian affine and $v$ is any unit minimum-eigenvalue eigenvector. This also defines a valid subgradient at repeated eigenvalues.

\paragraph{Synthetic families.}
Squared-affine and log-sum-exp systems use the corresponding primitives. The rotated-SOC family uses affine $P_i,Q_i,v_i,\ell_i$ and the convex residual
\[
g_i(y)=\left\|\begin{bmatrix}2P_i(y)\\2Q_i(y)\\v_i(y)-\ell_i(y)\end{bmatrix}\right\|-v_i(y)-\ell_i(y).
\]
A known feasible witness sets the margins. Reported comparisons use unit row scaling.

\subsection{SOC-MPC objective and constraints}
The decision stacks states $\xi_0,\ldots,\xi_H\in\R^6$ and controls $u_0,\ldots,u_{H-1}\in\R^2$. Each instance supplies the initial state, target position, and stored feasible tube centers $r_0,\ldots,r_H$. Equalities are $\xi_0=\xi_{\rm init}$ and $\xi_{t+1}=A_d\xi_t+B_du_t$. With $\xi_{\rm tar}$ containing the target position and zeros elsewhere, the objective is
\[
f_x(y)=\frac12\sum_{t=0}^H(\xi_t-\xi_{\rm tar})^\top Q_t(\xi_t-\xi_{\rm tar})
+\frac12\sum_{t=0}^{H-1}u_t^\top Ru_t,
\quad Q_H=5Q,\quad Q_t=Q\ (t<H).
\]
The $2H$ nonlinear residuals are
\[
g_t^{\rm tube}=\|D(\xi_t-r_t)\|-\rho\ (1\le t\le H),\qquad
g_t^{\rm input}=\|W_uu_t\|-u_{\max}\ (0\le t<H).
\]
Tube centers are instance data, distinct from targets and optimizer outputs. The dataset stores all matrices and radii. The $\Pi$net adapter receives the same affine SOC maps. The objective evaluator retains target-dependent constants.

\subsection{Convex QCQP objective and constraints}\label{app:qcqp-definition}
The input $x\in\R^{50}$ is the affine-equality right-hand side, and $y\in\R^{100}$ is the decision vector. Objective and constraint coefficients are fixed across the dataset, while $x$ varies:
\[
\min_y\ \tfrac12 y^\top Qy+p^\top y,\qquad
Cy=x,\qquad y^\top H_i y+a_i^\top y-h_i\le0,\quad i=1,\ldots,100.
\]
Both $Q$ and $H_i$ are positive semidefinite; the released dataset uses positive diagonal $H_i$. We directly query
\[
g_i(y)=y^\top H_i y+a_i^\top y-h_i,\qquad
\nabla g_i(y)=2H_i y+a_i.
\]
The inequality quadratic term has no factor $1/2$. Feasibility checks $\|Cy-x\|_\infty$ and $\max_i[g_i(y)]_+$ in original coordinates against $10^{-4}$. Gaps use the dataset's stored Gurobi reference objectives.

\subsection{ACOPF objective and constraints}
Both relaxations minimize generator cost, with per-unit generation $p_g,q_g$ and power base $S_{\rm base}$:
\[
f_x(y)=\sum_g\left[c_{2g}(S_{\rm base}p_g)^2+c_{1g}S_{\rm base}p_g+c_{0g}\right].
\]
Lifted coordinates include squared voltage magnitude $w_i$ and real/imaginary voltage-product components $w_{ij}^r,w_{ij}^i$. Branch-end flows $p_{ij},q_{ij}$ are affine in them. Bus balance is
\begin{align*}
\sum_{g\in\mathcal G_i}p_g-p_i^d-\sum_{j:(i,j)}p_{ij}-g_i^{\rm sh}w_i&=0,\\
\sum_{g\in\mathcal G_i}q_g-q_i^d-\sum_{j:(i,j)}q_{ij}+b_i^{\rm sh}w_i&=0.
\end{align*}
Affine inequalities bound generation, voltage, lifted products, and additional instance-specific envelopes. Angle residuals are $w_{ij}^i-\tan(\theta_{ij}^{\max})w_{ij}^r$ and $\tan(\theta_{ij}^{\min})w_{ij}^r-w_{ij}^i$. Nonlinear SOCP residuals are
\[
g_{ij}^{\rm thermal}=p_{ij}^2+q_{ij}^2-(s_{ij}^{\max})^2,
\qquad
g_{ij}^{\rm cone}=\left\|\begin{bmatrix}2w_{ij}^r\\2w_{ij}^i\\w_i-w_j\end{bmatrix}\right\|-w_i-w_j.
\]
Thermal limits are checked at both branch ends. The SDP adapter retains these and adds $g_c^{\rm PSD}(y)=-\lambda_{\min}(W_c(y))$ for every stored chordal PSD group. Each $W_c$ is Hermitian affine in the lifted coordinates; shared diagonal variables and duplicate off-diagonal components have consistency equalities.

\subsection{Scaling and numerical conventions}
The reported native-coordinate violation is $\max\{\|Cy-d(x)\|_\infty,\max_i[g_i(y,x)]_+\}$. Normalized residuals select and correct blocks, but feasibility uses native residuals. The norm denominator floor is $10^{-12}$ for synthetic selected rows and Torch SOC-MPC. ACOPF cone helpers floor the squared norm at the dtype's smallest positive normal value. SDP uses the eigensolver-selected eigenpair; its backward convention is in Appendix~\ref{app:training-protocol}.

\section{Additional Solver and Experiment Details}

\FloatBarrier
\subsection{Additional SOC-MPC controls}\label{app:soc-mpc-controls}
The identity-mobility predictor is still trained through the constraint layer. All configurations use the same split and correction budget, evaluate the epoch-100 checkpoint on 1,000 instances per horizon, and reach 100\% feasibility. The predictor rate is $3\times10^{-4}$ throughout. The main text uses validation-selected mobility-head rates: $3\times10^{-5}$ for H20 and $3\times10^{-4}$ for H40/H80. The lower H20 rate scales only the mobility-head update by 0.1 after Adam computes it. The table also includes an H20 sensitivity control with the same rate as the predictor.
\begin{table}[t]
\centering\setlength{\tabcolsep}{3pt}
\caption{SOC-MPC mobility controls. Gaps retain their signs; time is mean single-instance latency.}\label{tab:soc_mpc_mobility}
\begin{tabular}{rllrr}\toprule
$H$ & Mobility & Head rate & Gap (\%) & Time (ms)\\\midrule
20 & Identity & fixed & 0.19939 & 0.828\\
20 & Learned, common-rate control & $3\times10^{-4}$ & 0.33924 & 1.308\\
20 & Learned, main & $3\times10^{-5}$ & 0.19200 & 0.941\\
40 & Identity & fixed & 0.77589 & 0.947\\
40 & Learned, main & $3\times10^{-4}$ & 0.23561 & 1.061\\
80 & Identity & fixed & 0.94790 & 1.100\\
80 & Learned, main & $3\times10^{-4}$ & 0.16175 & 3.125\\\bottomrule
\end{tabular}
\end{table}
With validation-selected head rates, learned mobility lowers mean relative Gap by $3.7\%$, $69.6\%$, and $82.9\%$ versus identity across the three horizons, at higher latency. H20 has a small gain and is sensitive to the head rate: its common-rate control has $0.33924\%$ Gap, higher than identity. These comparisons concern mean per-instance relative Gap. H20 mean raw objective does not improve: identity and the main learned configuration give $10.21007$ and $10.21084$, respectively.

For the fixed-depth $\Pi$net control, each horizon reuses the same epoch-100 checkpoint, normalization, and projection parameters, varying only test projection depth $K$. The validation-selected depths were $200/300/400$ for H20/H40/H80. The main table selects $K=400$ at all horizons as representative configurations from this test sweep, using the smallest evaluated depth making the complete test set feasible. The original $1000/1000/400$ results are also retained below. Every test case runs exactly $K$ steps, without intermediate residual checks or early stopping. Timed inference includes the predictor, projection, and objective evaluation, excluding input transfer, compilation, warm-up, and the independent original-residual check. Every point covers the complete 1,000-case test set.
\begin{table}[t]
\centering\setlength{\tabcolsep}{3pt}
\caption{$\Pi$net fixed test depths on SOC-MPC (1,000 cases per horizon). Gap is the signed mean per-case relative objective difference; Max viol. is the maximum original constraint violation over all cases; time averages all instances. Feasibility uses tolerance $10^{-3}$. Rows $K=400$ at all three horizons are the main-table results. Each depth is evaluated separately; the original $K=(1000,1000,400)$ rows reuse the archived runs on the same checkpoints.}\label{tab:soc_mpc_pinet_fixed_depth}
\begin{tabular}{r r r r r r}\toprule
$H$ & Fixed $K$ & Gap (\%) & Max viol. & Feas. (\%) & Time (ms)\\\midrule
20 & 10 & 23.58977 & $6.13\times10^{-2}$ & 44.1 & 2.281\\
20 & 25 & 5.41482 & $3.65\times10^{-2}$ & 14.8 & 5.055\\
20 & 50 & 0.54972 & $1.29\times10^{-2}$ & 9.5 & 9.841\\
20 & 100 & 0.14579 & $2.75\times10^{-3}$ & 49.9 & 19.278\\
20 & 200 & 0.11107 & $1.36\times10^{-3}$ & 97.8 & 38.747\\
20 & 400 & 0.13016 & $6.21\times10^{-4}$ & 100.0 & 75.952\\
20 & 700 & 0.13173 & $2.83\times10^{-4}$ & 100.0 & 134.072\\
20 & 1,000 & 0.13206 & $1.68\times10^{-4}$ & 100.0 & 192.468\\\midrule
40 & 10 & 25.03752 & $8.82\times10^{-2}$ & 47.0 & 4.273\\
40 & 25 & 2.42901 & $4.70\times10^{-2}$ & 7.1 & 10.137\\
40 & 50 & 0.54179 & $1.24\times10^{-2}$ & 3.1 & 19.676\\
40 & 100 & 0.49964 & $2.59\times10^{-3}$ & 51.7 & 38.731\\
40 & 200 & 0.54620 & $1.42\times10^{-3}$ & 99.2 & 76.773\\
40 & 300 & 0.55698 & $1.06\times10^{-3}$ & 99.9 & 115.093\\
40 & 400 & 0.56725 & $8.96\times10^{-4}$ & 100.0 & 153.111\\
40 & 700 & 0.57937 & $4.33\times10^{-4}$ & 100.0 & 267.456\\
40 & 1,000 & 0.58085 & $2.08\times10^{-4}$ & 100.0 & 388.013\\\midrule
80 & 10 & 18.65364 & $7.24\times10^{-1}$ & 22.7 & 8.672\\
80 & 25 & 9.81001 & $8.35\times10^{-2}$ & 43.4 & 20.042\\
80 & 50 & 2.02172 & $4.94\times10^{-2}$ & 9.7 & 38.982\\
80 & 100 & 0.49334 & $9.24\times10^{-3}$ & 11.3 & 77.369\\
80 & 200 & 0.25986 & $2.34\times10^{-3}$ & 91.2 & 153.082\\
80 & 300 & 0.23428 & $1.05\times10^{-3}$ & 99.9 & 226.595\\
80 & 400 & 0.22984 & $8.41\times10^{-4}$ & 100.0 & 304.138\\\bottomrule
\end{tabular}
\end{table}
In this sweep, $K=400$ is the first evaluated fixed depth that makes all 1,000 test cases feasible at each horizon. The original validation-selected depths $200/300/400$ yield test feasibility of $97.8\%/99.9\%/100\%$. Shorter depths can have small Gaps while leaving outputs above tolerance, and feasibility does not increase monotonically with $K$. Figure~\ref{fig:soc_mpc_quality_latency} groups horizons by column and displays objective Gap and maximum constraint violation in the top and bottom rows, respectively.
\FloatBarrier

\paragraph{SnareNet stopping calibration.}
SnareNet reuses each horizon's final epoch-$100$ AdaRel-on checkpoint and changes only the test internal threshold and iteration cap. All $1{,}000$ validation cases evaluate thresholds $\{10^{-3},10^{-4},10^{-5},10^{-6}\}$ and caps $\{100,250,500,1000\}$. LM regularization stays at $0.01$, with the official residual-or-update stopping rule. Selection first maximizes the count meeting $10^{-3}$ original-residual tolerance, then minimizes mean repair iterations; remaining ties prefer smaller maximum violation, then a smaller cap, then a larger internal threshold.

Both horizons select internal threshold $10^{-4}$ and cap $500$, with full validation feasibility. Settings are fixed before evaluating the complete test set. Validation solves stop independently per instance; vectorization only accelerates evaluation and was checked against the official single-instance forward pass.

\begin{table}[t]
\centering\setlength{\tabcolsep}{3pt}
\caption{SnareNet SOC-MPC stopping controls, using one checkpoint and all $1{,}000$ test cases per horizon. Original settings use internal threshold $10^{-3}$ and cap $100$; validation-calibrated settings use $10^{-4}$ and $500$ and appear in the main table. Gap retains its sign; external feasibility always uses $10^{-3}$. Time is mean single-instance latency.}\label{tab:soc_mpc_snarenet_stopping}
\begin{tabular}{rlrrrr}\toprule
$H$ & Stopping & Gap (\%) & Max viol. & Feas. (\%) & Time (ms)\\\midrule
20 & Original & 0.01904 & $7.291\times10^{-3}$ & 77.9 & 20.534 \\
20 & Calibrated & 2.91321 & $2.642\times10^{-4}$ & 100.0 & 58.567 \\
40 & Original & 0.06786 & $4.385\times10^{-3}$ & 79.2 & 61.758 \\
40 & Calibrated & 1.76219 & $2.418\times10^{-4}$ & 100.0 & 175.515 \\
\bottomrule
\end{tabular}
\end{table}
Failures under the original settings primarily violate dynamics equalities. Calibration makes all test outputs feasible at both horizons, with more repair iterations and higher latency. The original settings' lower mean Gap includes infeasible outputs and does not measure feasible-solution quality for comparison with the main table.

\begin{table}[!htbp]
\caption{SOC-MPC main-table results at full precision; settings match Table~\ref{tab:soc_mpc}.}\label{tab:soc-mpc-full}
\centering\setlength{\tabcolsep}{3pt}
\begin{tabular}{@{}rlrrr@{}}\toprule
$H$ & Method & Gap (\%) & Feas. (\%) & Time (ms)\\\midrule
20 & Gurobi & -4.50e-5 & 100 & 20.809 \\
20 & MOSEK & -1.07e-6 & 100 & 4.974 \\
20 & DC3 & 5.90896 & 100 & \textbf{0.566} \\
20 & $\Pi$net & \textbf{0.13016} & 100 & 75.952 \\
20 & SnareNet & 2.91321 & 100 & 58.567 \\
20 & Ours & 0.19200 & 100 & 0.941 \\
40 & Gurobi & -4.61e-5 & 100 & 57.957 \\
40 & MOSEK & -1.84e-6 & 100 & 10.023 \\
40 & DC3 & 16.69631 & 100 & 1.810 \\
40 & $\Pi$net & 0.56725 & 100 & 153.111 \\
40 & SnareNet & 1.76219 & 100 & 175.515 \\
40 & Ours & \textbf{0.23561} & 100 & \textbf{1.061} \\
80 & Gurobi & -1.51e-4 & 100 & 59.254 \\
80 & MOSEK & -1.14e-4 & 100 & 20.470 \\
80 & DC3 & 13.34606 & 99.4 & 6.623 \\
80 & $\Pi$net & 0.22984 & 100 & 304.138 \\
80 & SnareNet & --- & --- & --- \\
80 & Ours & \textbf{0.16175} & 100 & \textbf{3.125} \\
\bottomrule\end{tabular}
\end{table}

\FloatBarrier
\subsection{Nonconvex optimization with affine constraints: depth and training controls}\label{app:pinet31-controls}

\begin{table}[t]
\caption{Complete affine-task results, including maximum original violation and batch timing. Five seeds each test $1{,}024$ instances; violation takes the maximum, other columns average seeds.}
\label{tab:affine-full}
\centering\setlength{\tabcolsep}{3pt}
\begin{tabular}{l r r r r r}
\toprule
Method & Gap (\%) & Max viol. & Feas. (\%) & \metrichead{Single time}{ms} & \metrichead{Batch time}{ms} \\
\midrule
DC3 & 7.89393 & 0.30938 & 99.824 & 3.043 & 3.809 \\
JAXopt & 0.38841 & $9.999\times10^{-4}$ & 100 & 7.226 & 111.679 \\
SnareNet & 6.16897 & $9.02\times10^{-4}$ & 100 & 3.041 & 61.441 \\
$\Pi$net ($K=50$) & 0.33117 & $1.02\times10^{-4}$ & 100 & 0.642 & 5.370 \\
Ours (identity mobility) & 0.24518 & $9.999\times10^{-4}$ & 100 & \textbf{0.410} & \textbf{1.864} \\
Ours (learned mobility) & \textbf{0.19032} & $9.995\times10^{-4}$ & 100 & 0.523 & 4.384 \\
\bottomrule
\end{tabular}
\end{table}
\paragraph{Inference timing and implementations.}
For each seed in Section~\ref{sec:affine}, single-instance timing synchronously repeats each of ten test inputs ten times and takes the median of the pooled $100$ measurements. The table averages these medians over five seeds. Batch-$1{,}024$ timing takes the median of ten complete-test-batch runs per seed, then averages over seeds. Timings include prediction and repair, excluding compilation and warm-up. DC3 uses its official PyTorch implementation; the other neural methods use JAX/XLA. Appendix~\ref{app:affine-protocol} gives training and repair settings. All five $50$-epoch JAXopt seeds use this ten-input timing protocol.

The main table uses $10^{-3}$ test tolerance for our method and fixed $50$-step inference for $\Pi$net. $\Pi$net trains with $50$ projection steps and implicit differentiation. This appendix also retains a test early-stop control that resumes the Douglas--Rachford state, checks original residuals every ten steps, and caps at $50$; the interval was selected on validation data. Forced continuation to $50$ steps reproduces the original fixed-depth output elementwise.

The fixed-depth sweep reuses the same five checkpoints with $K\in\{1,2,5,10,15,20,30,40,50,75,100\}$ and no intermediate residual checks. Each point covers every seed's complete test set. At $K=50$, quality metrics match the original results for each seed. All fixed-depth timings below come from this sweep. These depths are observations along the test curve; the main table retains the original $50$-step budget, and the additional early-stop control uses a validation-selected check interval. Table~\ref{tab:pinet31_stopping} reports the complete sweep.
\begin{table}[t]
\centering\setlength{\tabcolsep}{3pt}
\caption{Fixed-depth and early-stop $\Pi$net on the affine task. Gap and latency average five seeds; Max viol. is the maximum across all $5{,}120$ outputs. Feasibility uses $10^{-3}$, and Gap includes infeasible outputs.}\label{tab:pinet31_stopping}
\begin{tabular}{lrrrrr}\toprule
Test mode & Gap (\%) & Max viol. & Feas. (\%) & \metrichead{Single time}{ms} & \metrichead{Batch time}{ms}\\\midrule
$K=1$ & 29.77461 & $2.381\times10^{0}$ & 75.84 & 0.079 & 0.410 \\
$K=2$ & 4.14283 & $4.083\times10^{0}$ & 0 & 0.106 & 0.554 \\
$K=5$ & 0.02529 & $1.738\times10^{0}$ & 0 & 0.158 & 0.845 \\
$K=10$ & 0.25240 & $2.711\times10^{-1}$ & 0 & 0.211 & 1.335 \\
$K=15$ & 0.35811 & $7.469\times10^{-2}$ & 0.059 & 0.266 & 1.833 \\
$K=20$ & 0.32809 & $1.794\times10^{-2}$ & 0.02 & 0.312 & 2.317 \\
$K=30$ & 0.33120 & $2.866\times10^{-3}$ & 99.922 & 0.433 & 3.375 \\
$K=40$ & 0.33117 & $5.436\times10^{-4}$ & 100 & 0.533 & 4.381 \\
$K=50$ & 0.33117 & $1.020\times10^{-4}$ & 100 & 0.642 & 5.370 \\
$K=75$ & 0.33117 & $1.551\times10^{-6}$ & 100 & 0.855 & 7.798 \\
$K=100$ & 0.33117 & $2.360\times10^{-8}$ & 100 & 1.074 & 10.223 \\
\midrule
Check each 10, $10^{-3}$ & 0.33120 & $9.84\times10^{-4}$ & 100 & 0.494 & 3.990\\\bottomrule
\end{tabular}
\end{table}
At $K=30$, four outputs miss tolerance; at $K=40$, all are feasible. Further iterations leave Gap essentially unchanged while reducing maximum violation. Across five training seeds, Gap mean $\pm$ sample standard deviation is $0.388\pm0.194\%$ for JAXopt, $6.169\pm1.723\%$ for SnareNet, $0.331\pm0.191\%$ for fixed-$50$-step $\Pi$net, $0.245\pm0.020\%$ for identity mobility, and $0.190\pm0.027\%$ for learned mobility.

Ours uses $b=16$, five inner corrections, ten dual steps, and outer cap eight. The main configuration trains with tolerance $10^{-3}$. We also train to $10^{-6}$ and evaluate the same checkpoint at two test tolerances.
\begin{table}[t]
\centering\setlength{\tabcolsep}{3pt}
\caption{Training and test tolerance controls for our method on nonconvex optimization with affine constraints; all rows are 100\% feasible at their test tolerance.}\label{tab:pinet31_tolerances}
\begin{tabular}{lrrrrrr}\toprule
Mobility & Train tol & Test tol & Gap (\%) & Max viol. & \metrichead{Single time}{ms} & \metrichead{Batch time}{ms}\\\midrule
Identity & $10^{-3}$ & $10^{-3}$ & 0.24518 & $9.999\times10^{-4}$ & 0.410 & 1.864\\
Learned & $10^{-3}$ & $10^{-3}$ & 0.19032 & $9.995\times10^{-4}$ & 0.523 & 4.384\\
Identity & $10^{-6}$ & $10^{-3}$ & 0.24606 & $9.99\times10^{-4}$ & 0.341 & 1.678\\
Identity & $10^{-6}$ & $10^{-6}$ & 0.24636 & $9.96\times10^{-7}$ & 0.833 & 2.918\\
Learned & $10^{-6}$ & $10^{-3}$ & 0.19167 & $9.99\times10^{-4}$ & 0.543 & 4.908\\
Learned & $10^{-6}$ & $10^{-6}$ & 0.19215 & $9.999\times10^{-7}$ & 3.605 & 9.267\\\bottomrule
\end{tabular}
\end{table}
Gaps are similar across training tolerances. Tighter test tolerance substantially increases latency. Fixed-depth $\Pi$net has lower CV than $10^{-3}$ early stopping. Mean 50-epoch training times for $\Pi$net, identity mobility, and learned mobility are 14.35, 23.36, and 37.75 seconds. These are total wall-clock times including initial compilation and validation at every epoch, averaged over five training seeds.

\paragraph{SnareNet relaxation and training duration.}
SnareNet retains the official LM update and stops when either the residual or update reaches its threshold. Regularization is $0.01$, the cap is $100$, and the internal train and test threshold is $10^{-3}$. It uses the task's two $200$-unit hidden layers, Adam learning rate $0.01$, batch size $2{,}048$, and the corrected objective without a violation penalty. Each seed matches the predictor initialization and training order of the corresponding identity-mobility run. With AdaRel, initial slack is the per-constraint maximum violation of raw training predictions and decays linearly over the first $25$ epochs. Validation and test always use zero slack. The main table uses the final checkpoint after $50$ epochs with AdaRel enabled.

\begin{table}[t]
\centering\setlength{\tabcolsep}{3pt}
\caption{SnareNet relaxation and training-duration controls. Each row aggregates five seeds, each testing all $1{,}024$ instances. Gap gives mean $\pm$ sample standard deviation. Every configuration is $100\%$ feasible at $10^{-3}$. The $100$-epoch runs retain the $25$-epoch relaxation decay.}\label{tab:snarenet_affine_controls}
\begin{tabular}{rlrrrr}\toprule
Epochs & AdaRel & Gap (\%) & Max viol. & \metrichead{Single time}{ms} & \metrichead{Batch time}{ms}\\\midrule
50 & Off & $13.06420\pm2.36947$ & $9.61\times10^{-4}$ & 3.525 & 62.326\\
50 & On, main & $6.16897\pm1.72330$ & $9.02\times10^{-4}$ & 3.041 & 61.441\\
100 & Off & $8.72578\pm1.79889$ & $8.28\times10^{-4}$ & 4.465 & 65.159\\
100 & On & $4.72325\pm1.43125$ & $9.91\times10^{-4}$ & 2.943 & 58.483\\\bottomrule
\end{tabular}
\end{table}
AdaRel and longer training both improve objective quality on this task. With AdaRel enabled, extending training from $50$ to $100$ epochs lowers mean Gap from $6.169\%$ to $4.723\%$; this control uses additional training budget. The official repair's forward and gradient computations were checked against the TorchAX/JAX implementation. Single-instance timing uses ten test inputs with ten repeats each; batch timing uses all $1{,}024$ inputs. Both exclude compilation.

\FloatBarrier
\subsection{ACOPF training scope and mobility}\label{app:acopf-mobility}
All controls start from the same pretrained checkpoint, use the same split, 100-epoch budget, and test repair configuration. The main text reports post with full-predictor fine-tuning, post with solution-head fine-tuning, and joint training of the solution and mobility heads with a frozen trunk. Both post configurations use raw-output MSE and identity mobility at inference. The additional controls jointly train either the full predictor or only the solution head with identity mobility. Comparing the two joint configurations with a frozen trunk measures the benefit of adding mobility. Both identity-mobility joint variants use $\frac12(\mathcal L_{\rm raw}+\mathcal L_{\rm corr})$; learned mobility adds the regularizer in Appendix~\ref{app:acopf-protocol}.
\begin{table}[t]
\centering\setlength{\tabcolsep}{3pt}
\caption{ACOPF control mean gaps (\%). All models are 100\% feasible at $10^{-3}$; sample sizes are 200, 400, and 1,000 for 57, 118, and 300 buses.}\label{tab:acopf_mobility_controls}
\begin{tabular}{llrrrrr}\toprule
 & & \multicolumn{5}{c}{Mean gap (\%)}\\
\cmidrule(lr){3-7}
 & & \multicolumn{2}{c}{Post} & \multicolumn{3}{c}{Joint}\\
\cmidrule(lr){3-4}\cmidrule(lr){5-7}
Relax. & Buses & Full & Head & Identity full & Identity head & Learned head\\\midrule
SOCP & 57 & 0.05115 & 0.05142 & 0.04430 & 0.05013 & 0.05010\\
SOCP & 118 & 0.41878 & 0.18361 & 0.40576 & 0.17931 & 0.16658\\
SOCP & 300 & 0.09816 & 0.04123 & 0.09951 & 0.04130 & 0.03676\\
SDP & 57 & 0.07490 & 0.06420 & 0.07079 & 0.06408 & 0.06401\\
SDP & 118 & 0.37341 & 0.16452 & 0.35438 & 0.15406 & 0.14442\\
SDP & 300 & 0.05240 & 0.05438 & 0.05807 & 0.04181 & 0.03842\\\bottomrule
\end{tabular}
\end{table}
With the same frozen trunk, task loss, and correction budget, adding a learned mobility head lowers Gap in all six cases; 57-bus differences are small, and the learned branch adds mobility regularization. The precise SOCP300 matched gaps are 0.041305\% to 0.036762\% (11.0\% relative improvement), and SDP300 0.041806\% to 0.038418\% (8.1\%). Each ACOPF case uses one training seed; only the small task in Section~\ref{sec:affine} tests cross-seed consistency. Full-predictor joint training gives a SOCP57 gap of 0.04430\%, below the frozen-trunk, learned-mobility result of 0.05010\%. On SOCP300 and SDP300, its gaps of 0.09951\% and 0.05807\% exceed the corresponding full-FT post gaps of 0.09816\% and 0.05240\%.

The learned metric also varies by case. For $P_\psi(x)=\operatorname{diag}(p(x))$, we record the maximum test-sample condition number $\max_x\max_jp_j(x)/\min_jp_j(x)$. We also compute the standard deviation of each $p_j(x)$ across test samples and average it over coordinates. For SOCP/SDP57, these two statistics are $1.31/1.28$ and $0.000925/0.00121$; for SOCP/SDP118, $11.00/12.84$ and $0.02166/0.01710$; for SOCP/SDP300, $15.77/15.66$ and $0.03408/0.01499$. The 57-bus metrics stay close to identity and vary little across inputs, consistent with their nearly tied Gap controls.

\paragraph{Baseline stopping states on infeasible outputs.}
Table~\ref{tab:acopf-failure-budget} checks per-instance stopping records for seven baseline configurations. Every output that misses tolerance $10^{-3}$ in these configurations is finite and reaches the $5{,}000$-iteration cap; none exits earlier because of an internal objective, gradient, or update threshold. Maximum equality and inequality residuals use the original constraints over the complete test set. Main-text Gap and latency also include these failed outputs.
\begin{table}[t]
\centering\setlength{\tabcolsep}{3pt}
\caption{Infeasible ACOPF baseline outputs and iteration limits. ``At cap'' counts only infeasible outputs; maximum residuals include all test outputs.}\label{tab:acopf-failure-budget}
\begin{tabular}{llrrrr}\toprule
Task & Method & Failed / Total & At cap & Max eq. & Max ineq.\\\midrule
SOCP118 & $\Pi$net & 6/400 & 6 & $2.408\times10^{-12}$ & $1.008\times10^{-2}$\\
SOCP118 & FSNet & 12/400 & 12 & $1.855\times10^{-3}$ & $1.221\times10^{-2}$\\
SOCP118 & SnareNet & 47/400 & 47 & $1.398\times10^{-3}$ & $3.522\times10^{-2}$\\
SDP57 & SnareNet & 25/200 & 25 & $9.795\times10^{-4}$ & $2.278\times10^{-3}$\\
SDP118 & SnareNet & 295/400 & 295 & $1.501\times10^{-3}$ & $4.317\times10^{-2}$\\
SOCP300 & FSNet & 915/1,000 & 915 & $4.780\times10^{-3}$ & $8.692\times10^{-3}$\\
SOCP300 & $\Pi$net & 8/1,000 & 8 & $2.178\times10^{-11}$ & $3.892\times10^{-3}$\\\bottomrule
\end{tabular}
\end{table}

\FloatBarrier
\subsection{DC3 ACOPF attempts}\label{app:dc3-acopf}
We applied DC3 equality completion and partial correction to all constraints of SOCP57/118 and SDP57/118~\citep{donti2021dc3}. With five training correction steps, step $10^{-4}$, momentum 0.5, and equality/inequality penalties 5 each, all cases produced nonfinite forward values in the first training batch, before any optimizer update. Squared thermal violations dominated the correction gradient. After full-training-set calibration, fixed steps were $3\times10^{-11}$ for 57 buses and $3\times10^{-14}$ for 118 buses. At the original penalty, SOCP118 stopped after 562 epochs because of a nonfinite gradient; the other three completed 2,000 epochs but remained infeasible on test.

With total penalty $10^5$ split equally between equality and inequality, all four completed 2,000 epochs. AdamW used batch size 512, peak rate 0.05, 50-epoch warm-up and cosine decay; training correction stayed at five steps. Test cap was 250 and tolerance $10^{-3}$.
\begin{table}[t]
\centering\setlength{\tabcolsep}{3pt}
\caption{DC3 tests after step calibration and increased penalty; maximum violation uses native residuals.}\label{tab:dc3_attempts}
\begin{tabular}{lrrrr}\toprule
Case & Step & Epochs & Feasible tests & Max ineq. violation\\\midrule
SOCP57 & $3\times10^{-11}$ & 2000 & $0/200$ & 0.19304\\
SOCP118 & $3\times10^{-14}$ & 2000 & $0/400$ & 189.69979\\
SDP57 & $3\times10^{-11}$ & 2000 & $0/200$ & 0.13667\\
SDP118 & $3\times10^{-14}$ & 2000 & $0/400$ & 94.37722\\\bottomrule
\end{tabular}
\end{table}
Restoring $10^{-4}$ test step for SOCP57 reduced maximum violation to 0.16163 but still yielded $0/200$ feasible outputs. We do not interpret negative objective differences of infeasible outputs as solution quality. These attempts do not cover 300 buses.

\FloatBarrier
\subsection{MOSEK reference and timing audit}\label{app:mosek-audit}
MOSEK via PowerModels/JuMP/MosekTools solves the same relaxations used to generate labels: \texttt{SOCWRConicPowerModel} for SOCP and \texttt{SparseSDPWRMPowerModel} for SDP. All six cases use the paper's test split, original FP64 inputs, and original model (\texttt{model\_tau=0}). Primal and dual feasibility tolerances are $10^{-3}$, relative gap tolerance $10^{-8}$, and MOSEK default threads are retained. Each case warms up on two validation samples and solves all tests once. Main-text time is the mean wall time of \texttt{PowerModels.optimize\_model!}, including backend setup, solve, and extraction but excluding JuMP model construction, process start, warm-up, file I/O, and external checks.
\begin{table}[t]
\centering\setlength{\tabcolsep}{3pt}
\caption{MOSEK native-residual audit and mean model-construction time over 200/400/1,000 tests.}\label{tab:mosek_audit}
\begin{tabular}{lrrrr}\toprule
Relax. & Buses & Max equality & Max inequality & Build time (ms)\\\midrule
SOCP & 57 & $1.590\times10^{-7}$ & $2.557\times10^{-7}$ & 5.378\\
SOCP & 118 & $7.072\times10^{-4}$ & $7.370\times10^{-4}$ & 12.333\\
SOCP & 300 & $2.135\times10^{-4}$ & $1.216\times10^{-4}$ & 27.268\\
SDP & 57 & $6.393\times10^{-6}$ & $3.777\times10^{-7}$ & 16.649\\
SDP & 118 & $5.205\times10^{-6}$ & $4.890\times10^{-7}$ & 53.069\\
SDP & 300 & $9.826\times10^{-4}$ & $1.195\times10^{-4}$ & 331.634\\\bottomrule
\end{tabular}
\end{table}
All 3,200 solves returned OPTIMAL and pass the native $10^{-3}$ check; stored outputs reproduce the FP64 labels elementwise, giving zero gap to original labels. The common neural-method reference incurs only $3.65\times10^{-7}$\%--$2.41\times10^{-6}$\% mean gaps from front-end rounding, displayed as 0.00000\%. SDP objective scaling tries factors $1,0.3,0.1,0.03$ only after non-OPTIMAL status; every test succeeded at the first factor.

\FloatBarrier
\subsection{Unified TorchAX/JAX block ablation}\label{app:block-ablation}
The nine deterministic methods use one RTX 4090 with TorchAX/JAX over two scales, three convex families, and five paired problem seeds: 270 method--instance combinations and 666 timing trials. M-SSP and its paired dual PGD use another RTX 4090 on byte-identical inputs for 60 combinations and 180 successful timing trials. Native-coordinate equalities and inequalities are checked after stopping.

Single-row SKM refreshes nonlinear residuals after each normalized-max row correction; its frozen variant refreshes after five. Equality-LS fits selected boundary equations by pinv or lstsq; Block-pinv fits only violated selected rows. Violation-GD runs ten gradient steps on squared positive block violation. Block principal pivoting (BPP) solves the same halfspace QP as dual PGD; BPP-4 and BPP-16 cap each block solve at four and sixteen pivots, respectively. M-SSP follows the adaptive extrapolated minibatch subgradient method of~\citet[Section 4.3]{necoara2019minibatch}, sampling 128 rows with replacement and refreshing true subgradients each step. With $d_{k,i}=[g_i(z_k)]_+a_i/\|a_i\|^2$ and $\bar d_k=128^{-1}\sum_i d_{k,i}$, it uses $\beta_k=1.8/L_{128}^k$ where $L_{128}^k=\|\bar d_k\|^2/(128^{-1}\sum_i\|d_{k,i}\|^2)$. Zero denominators give a zero update. Each problem seed has one pre-generated random-index stream, reused across three timing repeats.

All methods use identity mobility, FP64, block size 128, and native maximum violation tolerance $10^{-3}$. Block methods relinearize per outer iteration and perform up to five inner updates; dual PGD and Violation-GD both use ten steps, zero initialization, and a matched spectral-bound step. Timings include oracle queries, gradients, selection, correction, scheduling, and stop checks, excluding compilation, warm-up, instance generation, equality preprocessing, and M-SSP index generation. Each run stops at tolerance, the $1{,}024$-outer-iteration cap, or a 0.5-second soft time budget checked after eight-iteration segments. Equality-LS's first segment itself takes roughly 1.04--1.08 seconds. Both Equality-LS variants fail to reach tolerance within the budget in all six groups. BPP-4 times out on four small squared-affine instances; single-row SKM reaches the iteration cap on five medium squared-affine and three medium log-sum-exp instances. All remaining method--instance combinations reach tolerance. Success counts use actual returned outputs.

For each instance, we take median repeated time, then median over five instances. A dagger marks stopping time on failures, not time-to-feasibility. Mean outer iterations weight the five instances equally; maximum violation is the maximum over all stopped outputs. Paired speedups are the median of per-instance alternative/dual ratios when both succeed on all five, and need not equal the ratio of group medians.

\begin{table}[!htbp]
\centering\setlength{\tabcolsep}{3pt}
\caption{Small synthetic systems $(n,n_{\rm eq},m)=(256,32,2048)$. Dagger: time includes failed outputs.}\label{tab:synthetic_block_alternatives}\label{tab:synthetic_small_full}
\begin{tabular}{lrrrrr}\toprule
Method & Success & Stop time (ms) & Mean outer & Max viol. & Paired ratio\\\midrule
\multicolumn{6}{l}{\textit{Squared affine}}\\
Single-row SKM & 5/5 & 75.606 & 589.0 & 9.952e-04 & 21.79\\
Single-row frozen & 5/5 & 29.745 & 117.8 & 9.516e-04 & 9.23\\
Equality-LS/pinv & 0/5 & 1045.233$^\dagger$ & 8.0 & 1.887e+02 & --\\
Equality-LS/lstsq & 0/5 & 1044.895$^\dagger$ & 8.0 & 8.198e+02 & --\\
Block-pinv & 5/5 & 547.114 & 5.4 & 8.586e-04 & 152.52\\
Violation-GD & 5/5 & 4.076 & 5.4 & 7.861e-04 & 1.17\\
BPP-4 & 1/5 & 536.639$^\dagger$ & 46.0 & 7.291e-03 & --\\
BPP-16 & 5/5 & 54.299 & 5.2 & 9.291e-04 & 16.18\\
Ours (dual PGD) & 5/5 & 3.470 & 5.6 & 5.836e-04 & 1.00\\\midrule
\multicolumn{6}{l}{\textit{Log-sum-exp}}\\
Single-row SKM & 5/5 & 102.326 & 515.2 & 9.957e-04 & 31.73\\
Single-row frozen & 5/5 & 34.383 & 103.6 & 9.989e-04 & 11.23\\
Equality-LS/pinv & 0/5 & 1045.260$^\dagger$ & 8.0 & 3.297e-01 & --\\
Equality-LS/lstsq & 0/5 & 1044.030$^\dagger$ & 8.0 & 3.123e-01 & --\\
Block-pinv & 5/5 & 417.574 & 4.0 & 8.954e-04 & 138.68\\
Violation-GD & 5/5 & 3.438 & 4.0 & 9.884e-04 & 1.12\\
BPP-4 & 5/5 & 30.054 & 3.8 & 7.489e-04 & 10.00\\
BPP-16 & 5/5 & 33.791 & 3.8 & 8.758e-04 & 10.57\\
Ours (dual PGD) & 5/5 & 3.047 & 4.0 & 9.844e-04 & 1.00\\\midrule
\multicolumn{6}{l}{\textit{Rotated SOC}}\\
Single-row SKM & 5/5 & 48.903 & 326.4 & 9.844e-04 & 20.74\\
Single-row frozen & 5/5 & 18.444 & 63.8 & 9.998e-04 & 7.82\\
Equality-LS/pinv & 0/5 & 1043.614$^\dagger$ & 8.0 & 3.270e-01 & --\\
Equality-LS/lstsq & 0/5 & 1045.805$^\dagger$ & 8.0 & 4.724e-01 & --\\
Block-pinv & 5/5 & 322.623 & 3.4 & 7.286e-04 & 135.42\\
Violation-GD & 5/5 & 2.790 & 3.0 & 6.757e-04 & 1.06\\
BPP-4 & 5/5 & 23.321 & 3.0 & 8.972e-04 & 9.89\\
BPP-16 & 5/5 & 29.333 & 3.0 & 8.601e-04 & 12.15\\
Ours (dual PGD) & 5/5 & 2.363 & 3.0 & 9.225e-04 & 1.00\\\bottomrule
\end{tabular}
\end{table}

\begin{table}[!htbp]
\centering\setlength{\tabcolsep}{3pt}
\caption{Medium synthetic systems $(n,n_{\rm eq},m)=(512,64,4096)$. Dagger: time includes failed outputs.}\label{tab:synthetic_medium_full}
\begin{tabular}{lrrrrr}\toprule
Method & Success & Stop time (ms) & Mean outer & Max viol. & Paired ratio\\\midrule
\multicolumn{6}{l}{\textit{Squared affine}}\\
Single-row SKM & 0/5 & 182.660$^\dagger$ & 1024.0 & 1.092e-02 & --\\
Single-row frozen & 5/5 & 107.956 & 242.2 & 9.847e-04 & 21.93\\
Equality-LS/pinv & 0/5 & 1066.688$^\dagger$ & 8.0 & 2.561e+02 & --\\
Equality-LS/lstsq & 0/5 & 1069.191$^\dagger$ & 8.0 & 2.539e+02 & --\\
Block-pinv & 5/5 & 642.880 & 5.6 & 6.650e-04 & 136.56\\
Violation-GD & 5/5 & 5.004 & 5.4 & 7.590e-04 & 1.14\\
BPP-4 & 5/5 & 39.231 & 5.6 & 9.481e-04 & 8.66\\
BPP-16 & 5/5 & 43.175 & 5.6 & 8.159e-04 & 9.38\\
Ours (dual PGD) & 5/5 & 4.759 & 5.6 & 9.912e-04 & 1.00\\\midrule
\multicolumn{6}{l}{\textit{Log-sum-exp}}\\
Single-row SKM & 2/5 & 243.191$^\dagger$ & 999.6 & 2.025e-03 & --\\
Single-row frozen & 5/5 & 100.015 & 212.8 & 9.987e-04 & 27.35\\
Equality-LS/pinv & 0/5 & 1067.914$^\dagger$ & 8.0 & 3.295e-02 & --\\
Equality-LS/lstsq & 0/5 & 1066.489$^\dagger$ & 8.0 & 2.718e-01 & --\\
Block-pinv & 5/5 & 476.396 & 4.0 & 7.853e-04 & 124.72\\
Violation-GD & 5/5 & 4.257 & 4.0 & 9.250e-04 & 1.11\\
BPP-4 & 5/5 & 28.154 & 4.0 & 6.455e-04 & 7.49\\
BPP-16 & 5/5 & 28.204 & 4.0 & 4.301e-04 & 7.84\\
Ours (dual PGD) & 5/5 & 3.822 & 4.0 & 7.755e-04 & 1.00\\\midrule
\multicolumn{6}{l}{\textit{Rotated SOC}}\\
Single-row SKM & 5/5 & 152.090 & 670.2 & 9.880e-04 & 50.00\\
Single-row frozen & 5/5 & 66.106 & 137.4 & 9.992e-04 & 20.65\\
Equality-LS/pinv & 0/5 & 1069.885$^\dagger$ & 8.0 & 2.928e-01 & --\\
Equality-LS/lstsq & 0/5 & 1070.832$^\dagger$ & 8.0 & 9.302e-02 & --\\
Block-pinv & 5/5 & 382.370 & 3.2 & 6.725e-04 & 115.88\\
Violation-GD & 5/5 & 3.585 & 3.0 & 8.009e-04 & 1.12\\
BPP-4 & 5/5 & 23.349 & 3.0 & 9.840e-04 & 7.29\\
BPP-16 & 5/5 & 24.068 & 3.0 & 6.860e-04 & 7.28\\
Ours (dual PGD) & 5/5 & 3.201 & 3.0 & 9.269e-04 & 1.00\\\bottomrule
\end{tabular}
\end{table}

\begin{table}[t]
\centering\setlength{\tabcolsep}{3pt}
\caption{M-SSP versus same-device dual PGD. Both succeed on all five seeds per group. Ratios are medians of paired instance ratios.}\label{tab:mssp_paired}
\begin{tabular}{llrrrr}\toprule
 & & \multicolumn{2}{c}{Time (ms)} & & \\
\cmidrule(lr){3-4}
Scale & Family & M-SSP & Dual PGD & Ratio & M-SSP outer\\\midrule
Small & Squared affine & 16.916 & 3.617 & 4.88 & 140.2\\
& Log-sum-exp & 23.284 & 3.257 & 7.43 & 160.2\\
& Rotated SOC & 11.842 & 2.315 & 5.11 & 95.0\\
Medium & Squared affine & 41.322 & 4.898 & 9.44 & 297.6\\
& Log-sum-exp & 67.219 & 4.303 & 15.90 & 362.2\\
& Rotated SOC & 73.798 & 4.448 & 19.56 & 272.4\\\bottomrule
\end{tabular}
\end{table}
All 180 M-SSP/dual paired trials reach tolerance; maximum native inequality violation is $9.9125\times10^{-4}$ and equality residual is below $10^{-14}$. M-SSP computes averaged row corrections, not a joint minimum-norm block correction. Its ratios use its own same-device dual control and are not obtained by mixing the two GPU runs.

\FloatBarrier
\subsection{Convex QCQP statistics}\label{app:qcqp-details}
Section~\ref{sec:qcqp} uses SnareNet's public data with $100$ variables, $100$ convex quadratic inequalities, and $50$ affine equalities. All seven main-table configurations test the complete set of $833$ instances. Joint configurations train their predictors through their respective repairs; the three postprocessors share the same predictor weights and raw outputs. Main-table results use seed $456$ and the final epoch-$1{,}500$ checkpoint; Appendix~\ref{app:qcqp-training} gives training settings.

\begin{table}[t]
\centering\setlength{\tabcolsep}{3pt}
\caption{Convex QCQP absolute objective gaps and maximum original residuals. Absolute gap is $f_x(y)-f_x^{\rm ref}$; its geometric mean and maximum cover all test instances. All gaps are positive. All configurations are $100\%$ feasible at $10^{-4}$.}\label{tab:qcqp_details}
\begin{tabular}{llrrrr}\toprule
 & & \multicolumn{2}{c}{Absolute gap} & & \\
\cmidrule(lr){3-4}
Mode & Method & Geom. mean & Max & Max ineq. & Max eq.\\\midrule
Joint & DC3 & 8.7725242 & 13.8658313 & $9.980\times10^{-7}$ & $1.293\times10^{-14}$\\
Joint & SnareNet & 0.0202505 & 0.0749286 & $2.16\times10^{-5}$ & $3.33\times10^{-5}$\\
Joint & $\Pi$net ($K=200$) & 0.00194093 & 0.0616385 & $1.347\times10^{-8}$ & $1.205\times10^{-14}$\\
Joint & Ours & 0.00130975 & 0.0165973 & $1.00\times10^{-6}$ & $8.95\times10^{-15}$\\
Post & SnareNet & 0.500654 & 1.038013 & $1.42\times10^{-5}$ & $1.85\times10^{-5}$\\
Post & $\Pi$net ($K=400$) & 0.497508 & 1.113269 & $7.742\times10^{-9}$ & $1.046\times10^{-14}$\\
Post & Ours & 0.497585 & 1.113269 & $1.00\times10^{-6}$ & $4.86\times10^{-15}$\\\bottomrule
\end{tabular}
\end{table}

The main text uses the arithmetic mean of nonnegative per-instance relative Gap; this table retains SnareNet's absolute-gap statistics. Feasibility uses the original quadratic-inequality and affine-equality residuals. SnareNet has internal threshold $10^{-6}$ but may also stop on small updates or its iteration cap, so outputs need not meet $10^{-6}$ original-residual tolerance. The main table consistently evaluates feasibility at $10^{-4}$.

The four SnareNet and \method{} configurations run sequentially on one RTX 4090 for three full single-instance timing passes. $\Pi$net and DC3 are timed separately on another RTX 4090 with the same test set, precision, and repetition count. These passes repeat inference with one checkpoint, not training seeds. All three postprocessors have identical predictor checkpoint hashes. Every $\Pi$net sweep point restarts repair from the same raw predictions and runs its fixed number of steps without intermediate residual checks.

\paragraph{DC3 training budget and test correction.}
We retain DC3's released $1{,}000$-epoch budget and continue the same run to $1{,}500$ epochs. Table~\ref{tab:qcqp_dc3_controls} compares both budgets and test calibration. The last two rows use the same checkpoint and differ only in test correction settings. The released test configuration applies ten fixed steps, then at most ten extra steps, all with step size $10^{-12}$ and internal tolerance $10^{-4}$ for the extra phase.

\begin{table}[t]
\centering\setlength{\tabcolsep}{3pt}
\caption{DC3 training-budget and test-correction controls. Each row covers all $833$ test inputs. Feasibility uses $10^{-4}$ original-residual tolerance; Gap and latency include failures. Calibration uses validation data only. The last row is the main-table configuration.}\label{tab:qcqp_dc3_controls}
\begin{tabular}{rlrrrr}\toprule
Epochs & Test correction & Gap (\%) & Feas. (\%) & Max viol. & Time (ms)\\\midrule
$1{,}000$ & Released & 78.19945 & 97.719 & $5.080\times10^{-1}$ & 0.372\\
$1{,}500$ & Released & 66.90930 & 96.279 & $1.752\times10^{0}$ & 0.378\\
$1{,}500$ & Calibrated & 66.91845 & 100 & $9.980\times10^{-7}$ & 0.600\\\bottomrule
\end{tabular}
\end{table}

Test calibration reduces infeasible outputs from $31$ to zero for the final checkpoint, but mean Gap remains about $66.9\%$. Objective quality is thus the main limitation of this configuration: extra test correction restores feasibility without improving its large Gap. Calibrated extra correction averages $7.94$ steps and takes at most $709$, in addition to the first ten fixed steps.

\paragraph{Fixed-depth $\Pi$net inference.}
Joint uses the final checkpoint trained with $200$ fixed projection steps; post uses the shared soft-constraint predictor. Each curve holds its validation-selected $\sigma$ fixed and evaluates all $833$ test inputs. Table~\ref{tab:qcqp_fixed_depth} reports nonnegative relative Gap, maximum original violation, $10^{-4}$ feasibility, and mean single-instance latency over three passes.

\begin{table}[t]
\centering\setlength{\tabcolsep}{3pt}
\caption{Complete fixed-depth $\Pi$net sweep on convex QCQP. Each row tests all $833$ instances. Max viol. is the largest original residual; feasibility uses $10^{-4}$. Gap includes infeasible outputs. The joint $K=200$ and post $K=400$ rows also appear in the main table.}\label{tab:qcqp_fixed_depth}
\begin{tabular}{llrrrr}\toprule
Mode & $K$ & Gap (\%) & Max viol. & Feas. (\%) & Time (ms)\\\midrule
Joint & 5 & 0.00000 & $1.996\times10^{0}$ & 0 & 0.330 \\
 & 10 & 0.00032 & $5.884\times10^{-1}$ & 0 & 0.465 \\
 & 20 & 0.06495 & $1.119\times10^{-1}$ & 0 & 0.753 \\
 & 30 & 0.08713 & $3.960\times10^{-2}$ & 0 & 1.057 \\
 & 50 & 0.04274 & $8.110\times10^{-3}$ & 0 & 1.639 \\
 & 75 & 0.04031 & $8.920\times10^{-4}$ & 58.583 & 2.422 \\
 & 100 & 0.04014 & $9.753\times10^{-5}$ & 100 & 3.175 \\
 & 150 & 0.04013 & $1.151\times10^{-6}$ & 100 & 4.645 \\
 & 200 & 0.04013 & $1.347\times10^{-8}$ & 100 & 6.126 \\
 & 300 & 0.04013 & $1.622\times10^{-12}$ & 100 & 9.007 \\
 & 400 & 0.04013 & $1.227\times10^{-14}$ & 100 & 11.900 \\
 & 600 & 0.04013 & $1.185\times10^{-14}$ & 100 & 17.624 \\
 & 800 & 0.04013 & $1.252\times10^{-14}$ & 100 & 23.369 \\
\midrule
Post & 5 & 70.62898 & $2.665\times10^{-15}$ & 100 & 0.293 \\
 & 10 & 42.65650 & $4.274\times10^{-15}$ & 100 & 0.446 \\
 & 20 & 19.83004 & $5.860\times10^{-2}$ & 58.103 & 0.744 \\
 & 30 & 12.02816 & $3.500\times10^{-2}$ & 54.982 & 1.028 \\
 & 50 & 6.54757 & $1.303\times10^{-2}$ & 48.739 & 1.615 \\
 & 75 & 4.73516 & $6.318\times10^{-3}$ & 46.098 & 2.415 \\
 & 100 & 4.20975 & $3.274\times10^{-3}$ & 49.46 & 3.188 \\
 & 150 & 3.97701 & $5.280\times10^{-4}$ & 92.557 & 4.616 \\
 & 200 & 3.94676 & $5.806\times10^{-5}$ & 100 & 6.143 \\
 & 300 & 3.94167 & $8.317\times10^{-7}$ & 100 & 9.027 \\
 & 400 & 3.94154 & $7.742\times10^{-9}$ & 100 & 11.880 \\
 & 600 & 3.94154 & $1.080\times10^{-11}$ & 100 & 17.626 \\
 & 800 & 3.94154 & $1.088\times10^{-14}$ & 100 & 23.354 \\
\bottomrule
\end{tabular}
\end{table}
At joint $K=5$, all outputs are infeasible and clipping negative per-instance relative gaps gives zero mean Gap; this does not represent feasible objective quality. Post $K=5,10$ is fully feasible but has large Gap, and feasibility is nonmonotone in subsequent iterations. Both cases motivate reporting objective quality together with original residuals.

\end{document}